\documentclass[11pt]{article}

\usepackage{amsmath}
\usepackage{amssymb}
\usepackage{amsthm}
\usepackage{graphicx}
\usepackage[margin=1in]{geometry}
\usepackage{booktabs}
\usepackage{xcolor}
\usepackage{placeins}
\usepackage[round,authoryear]{natbib}
\usepackage{hyperref}
\hypersetup{
  colorlinks = true,
  linkcolor  = {blue!55!black},
  citecolor  = {blue!55!black},
  urlcolor   = {blue!55!black},
  pdftitle   = {Rank-Conditioned Sample Reuse for the Plackett-Luce Best-of-K Objective},
  pdfauthor  = {Melveena Jolly and Midhun Xavier}
}

\allowdisplaybreaks

\theoremstyle{plain}
\newtheorem{theorem}{Theorem}
\newtheorem{proposition}{Proposition}
\newtheorem{lemma}{Lemma}
\newtheorem{corollary}{Corollary}
\theoremstyle{definition}
\newtheorem{observation}{Observation}
\theoremstyle{remark}
\newtheorem{remark}{Remark}

\newcommand{\E}{\mathbb{E}}

\newcommand{\ind}{\mathbf{1}}
\newcommand{\pool}{\mathrm{pool}}
\newcommand{\WOR}{\mathrm{WOR}}
\newcommand{\Jhat}{\hat{J}}
\newcommand{\JhatQ}{\hat{J}_{Q}}
\newcommand{\JW}{J^{\WOR}_{K}}
\newcommand{\JIID}{J^{\mathrm{iid}}_{K}}
\newcommand{\PLWOR}{\mathrm{PL}\text{-}\WOR}
\DeclareMathOperator{\sgop}{sg}
\DeclareMathOperator{\Gumbel}{Gumbel}
\DeclareMathOperator{\Expo}{Exp}

\title{Rank-Conditioned Sample Reuse\\
for the Plackett--Luce Best-of-$K$ Objective}

\author{%
  Melveena Jolly \\
  \textit{Independent Researcher} \\
  \texttt{melveenajollyk@gmail.com}
  \and
  Midhun Xavier \\
  \textit{Independent Researcher} \\
  \texttt{midhunxavier@outlook.com}
}
\date{July 2026}

\begin{document}
\maketitle

\begin{abstract}
We study the coupled objective $\JW=\E_{S\sim\PLWOR_K}\big[\max_{i\in S}R_i\big]$:
the expected maximum reward of a size-$K$ Plackett--Luce draw \emph{without
replacement}, the law of Gumbel-Top-$K$ / Stochastic Beam Search decoding. This
estimand differs from the conventional i.i.d.\ objective
$\JIID=\E[\max_{i\le K}R_i]$ ($R_i$ independent) targeted by existing sample-reuse
Max@$K$ estimators, and reusing their i.i.d.\ weights under the coupled sampler is
provably biased: we exhibit a closed-form three-item instance with
$\E[\widehat g_{\mathrm{iid}}]=\tfrac45\,\nabla_\theta\JW$ exactly (pass@$K$
\emph{under the coupled sampler} is the binary-reward special case of $\JW$).
Generic joint-score REINFORCE is already unbiased for $\JW$; what it lacks is sample
reuse. Our contribution is to instantiate standard rank-conditioned
Horvitz--Thompson estimation for the $\JW$ subset total: from one Gumbel-Top-$n$
pool ($n>K$) and its observed priority threshold we build an estimator that reuses
all $\binom nK$ embedded $K$-subsets, unbiased for $\JW$ (Theorem~\ref{thm:unbiased})
with an unbiased exact score-function surrogate gradient
(Proposition~\ref{prop:gradient}), together with a reward-sorted Max-specific
dynamic program that collapses the $\binom nK$-term subset sum (each term carrying a
$K!$-cost set probability) \emph{exactly} to a one-dimensional integral. A fixed-$Q$
quadrature evaluation costs $O(n\log n+nKQ)$ arithmetic operations
(Theorem~\ref{thm:tractable}); it is numerically, not algebraically, exact, and we
certify no $\epsilon$-approximation rate. Each nonzero degree-$K$ Horvitz--Thompson
term has finite second moment exactly when $n\ge2K$, and under the same
finite-support interior assumptions the \emph{full} surrogate gradient has finite
second moment whenever $n\ge2K$ (Proposition~\ref{prop:gradmoment}); sharpness at
the gradient level remains open. The construction recovers the classical single-item
priority-sampling estimator at $K=1$. All quantities require only the values and
differentiable computation graphs of the $n{+}1$ drawn items' probabilities, so
finite structured sequence policies sampled by exact Stochastic Beam Search are
covered (Corollary~\ref{cor:sequence}). A certified finite-$Q$ quadrature error
bound and countably infinite support remain open.
\end{abstract}

\section{Introduction}
\label{sec:intro}

Reinforcement learning increasingly optimizes not the \emph{mean} reward but a
\textbf{best-of-$K$} reward, the quantity a deployed system actually realizes when
it samples $K$ candidates and keeps the winner. Two best-of-$K$ estimands must be
separated from the first line. The \textbf{i.i.d.} objective,
\[
\JIID=\E\big[\max_{i\le K}R_i\big],\qquad R_1,\dots,R_K\ \text{i.i.d.\ from }p_\theta,
\]
is the standard Max@$K$ / pass@$K$ target and now an active line of work: PKPO
\citep{pkpo2025}, RSPO \citep{rspo2025}, MaxPO \citep{maxpo2026}, OrderGrad
\citep{ordergrad2026}, and related methods derive unbiased policy gradients for it
by reusing an i.i.d.\ batch. This note is about the \textbf{coupled} objective
\[
\JW=\E_{S\sim\PLWOR_K}\Big[\max_{i\in S}R_i\Big],
\]
the expected maximum of a size-$K$ Plackett--Luce draw \emph{without replacement}:
the estimand realized when a neural combinatorial optimization (NCO) solver decodes
$K$ \emph{distinct} tours in one Gumbel-Top-$K$ / Stochastic Beam Search draw and
returns the shortest, or an LLM system samples $K$ distinct generations in one such
draw and keeps the one a verifier prefers (pass@$K$ \emph{under the coupled sampler}
is the binary-reward special case). The two estimands differ, and estimators tuned
to one are generally biased for the other.

\paragraph{The gap.}
The cited sample-reuse estimators assume the $K$ samples are drawn \textbf{i.i.d.}:
independently, with replacement. Duplicate candidates contribute nothing to the max, so
one important design choice is to bias the sampler \emph{toward distinct candidates on
purpose}: Gumbel-Top-$k$ / Stochastic Beam Search \citep{kool2019sbs}, deduplication before verification,
diverse (penalty-based or deterministic) beam variants, and QMC-coupled batches. These
are \emph{not} one sampling class: only Gumbel-Top-$n$/SBS has the Plackett--Luce /
exponential-race structure our theorems use; deduplicated i.i.d.\ sampling induces a
different conditional law, diverse beam search need not be a PL draw, and QMC batches
preserve marginals with correlated draws. The latter three motivate the broader
correlated-generation problem but are \emph{not covered} by our results.
\citet{shi2020incremental} showed the computational prize on TSP: $100$ distinct
samples can outperform $1280$ i.i.d.\ samples in a duplicate-heavy regime. What is
\emph{missing} is not an unbiased gradient per se (the generic score-function
estimator for a Plackett--Luce draw is already unbiased for any WOR-draw functional,
with the control variates of \citet{gadetsky2020plackett} providing a natural
variance-reduced baseline), but a \textbf{sample-reuse} estimator of $\JW$ that
reuses an $n>K$ WOR pool through closed-form order-statistic (inclusion-probability)
weights, the without-replacement analogue of the i.i.d.\ order-statistic estimators of
PKPO / RSPO / OrderGrad. Under WOR the joint inclusion probability of a $K$-subset is not
the product of marginals, so a naive Horvitz--Thompson reweight of those i.i.d.\
estimators fails; the rank-conditioning of Theorem~\ref{thm:unbiased} is what
restores unbiasedness. MaxPO's limitations section \citep{maxpo2026} explicitly leaves
``settings with correlated generations'' to future work; sampler-level WOR coupling is
one instance of that broad class.

\paragraph{A clarification that prevents a conflation.}
The ``without replacement'' in our setting is \emph{sampler-level, generative}
coupling: the policy itself emits $K$ distinct, negatively correlated items in a
single Gumbel-Top-$K$ / Stochastic Beam Search / Plackett--Luce draw. This is
categorically different from the \emph{within-batch} WOR used by the i.i.d.\ Max@$K$
estimators \citep{pkpo2025,rspo2025,maxpo2026,ordergrad2026,bonw2025}, which first draw
$n$ responses \textbf{i.i.d.\ with replacement} and only then sub-sample size-$K$
subsets of that \emph{already realized} batch without replacement (a $U$-statistic /
variance-reduction device, not a change of sampler). The two induce different
objectives, $\JIID$ there and $\JW$ here: ours is the expectation under a coupled
draw. Without this one-line distinction a reader may read our result as already
solved; it is not.

\paragraph{This note.}
We address that sample-reuse gap for one important coupled sampler: Gumbel-Top-$n$ /
Plackett--Luce sampling with an observed priority threshold and conditionally factorized
inclusion events. We do not cover deduplicated i.i.d.\ sampling, penalty-based diverse
beam search, or QMC coupling. The central claim, stated narrowly: \emph{we instantiate
standard rank-conditioned Horvitz--Thompson estimation
\citep{duffield2007priority,cohen2008bottomk} for the Plackett--Luce best-of-$K$
subset total, and derive a Max-specific dynamic program that removes the combinatorial
subset sum}. Concretely:
\begin{itemize}
\item a rank-conditioned Horvitz--Thompson construction for the Plackett--Luce
best-of-$K$ set functional, with a finite-support unbiasedness proof
(Theorem~\ref{thm:unbiased}) and an unbiased exact score-function policy gradient
(Proposition~\ref{prop:gradient});
\item a Max-specific reward-sorted dynamic program that removes the $\binom nK$-term
subset sum \emph{exactly}, leaving one one-dimensional integral evaluated by
quadrature in $O(n\log n+nKQ)$ arithmetic operations at fixed $Q$
(Theorem~\ref{thm:tractable}); we claim the exact combinatorial collapse plus a
fixed-node arithmetic cost, not an $\epsilon$-approximation rate;
\item second-moment structure: each nonzero degree-$K$ Horvitz--Thompson term has
finite second moment exactly when $n\ge2K$ (Remark~\ref{rmk:variance}), and the
\emph{full} surrogate gradient has finite second moment whenever $n\ge2K$
(Proposition~\ref{prop:gradmoment}), making $n\ge2K$ a certified sufficient
training-stability condition at the gradient level (gradient-level sharpness remains
open); and
\item sharp scope diagnostics: standard conditionally independent multi-start rollouts
need no coupling correction (Observation~\ref{obs:pomo}), while genuine sampler-level
coupling provably does: a closed-form three-item counterexample with
$\E[\widehat g_{\mathrm{iid}}]=\tfrac45\,\nabla_\theta\JW$ exactly
(Observation~\ref{obs:coupling}).
\end{itemize}
The estimator and sampler score need only the values and differentiable computation
graphs of the drawn items' probabilities together with the normalization
$\sum_y p_\theta(y)=1$ (Corollaries~\ref{cor:poolonly} and~\ref{cor:sequence}). Thus
the finite-policy theory covers both flat categorical supports and finite structured
sequence supports sampled using exact SBS. Countably infinite or unbounded
variable-length support remains open (\S\ref{sec:discussion}).
Section~\ref{sec:numerics} supplies numerical certificates, not an application
benchmark. This is a theory-and-certification note; a pre-registered multi-seed NCO
benchmark is deferred to the full paper.

\section{Preliminaries}
\label{sec:prelim}

\paragraph{Objective.}
Items (tours, generations) $1,\dots,M$ form a fixed, finite support and have
full-support policy probabilities $p_i=p_\theta(i)>0$ (so $\sum_i p_i=1$,
$\phi_i=\log p_i$) and fixed, $\theta$-independent rewards $R_i$. State constraints
may remove invalid sequences provided the remaining support is fixed in $\theta$. A
fresh WOR draw of size $K$ is a Plackett--Luce (PL) sequential draw (pick
$\propto p$, renormalize, repeat, $K$ times; written $S\sim\PLWOR_K$); let
$P_{\WOR}(S)$ be the probability its outcome \emph{set} is $S$. The coupled
best-of-$K$ objective of this note is
\begin{equation}
\JW=\E_{S\sim\PLWOR_K}\Big[\max_{i\in S}R_i\Big]
=\sum_{|S|=K}P_{\WOR}(S)\,\max_{i\in S}R_i .
\label{eq:obj}
\end{equation}
Three sizes recur and are kept apart notationally: $K$ is the coupled draw size (one
draw of $K$ distinct items realizes the best-of-$K$), $n\ge K$ is the size of the
larger pool the estimator reuses, and $M$ is the support size. We reserve
$\JIID=\E[\max_{i\le K}R_i]$ ($R_i$ i.i.d.\ from $p_\theta$) for the conventional
objective of the i.i.d.\ Max@$K$ literature; $\JIID\ne\JW$ in general, and every
claim below concerns $\JW$.

\paragraph{The coupled sampler.}
A Gumbel-Top-$n$ / SBS sampler perturbs each item, $G_i=\phi_i+\Gumbel(0,1)$ with
$\phi_i=\log p_i$, draws the pool $=$ top-$n$ by $G$ (with $n\ge K$, and $n>K$ for genuine \emph{sample
reuse}, so the pool contains more than one size-$K$ subset), and records the priority
threshold $\kappa=G_{(n+1)}$. This is the structure the proofs use: independent
Gumbel/exponential clocks, an \emph{observed} $(n{+}1)$-st priority threshold, and
inclusion events that factorize conditionally on that threshold; this is more than
closed-form marginal inclusion probabilities alone. Its threshold-conditioned
(rank-conditioned) per-item inclusion probability is
\begin{equation}
q_i(\kappa)=\Pr(G_i>\kappa)=1-\exp\!\big(-e^{\phi_i-\kappa}\big).
\label{eq:incl}
\end{equation}

\paragraph{The estimator.}
Reusing all $\binom nK$ size-$K$ sub-draws of the pool, corrected by the inclusion
probabilities:
\begingroup\renewcommand{\theHequation}{eqstar}
\begin{equation}
\Jhat(\pool,\kappa)=\sum_{\substack{S\subseteq\pool\\|S|=K}}
\frac{P_{\WOR}(S)}{\prod_{i\in S}q_i(\kappa)}\,\max_{i\in S}R_i .
\tag{$\star$}
\label{eq:star}
\end{equation}
\endgroup
The exact theory uses the score-function surrogate
$\mathcal L=-\big(\Jhat+\sgop(\Jhat)\log f\big)$, whose gradient
$-\big(\nabla_\theta\Jhat+\Jhat\,\nabla_\theta\log f\big)$ is the negative of the ascent
gradient: the pathwise term plus the sampler
score ($f$ is the sampler density and $\sgop$ the stop-gradient). Here the realized
priority $\kappa$ is treated as a \emph{fixed} draw: both $\Jhat$ and $\log f$ are
differentiated only through their explicit $\theta$-dependence (through $P_{\WOR}$ and
$q_i$), not through $\kappa$, exactly as in the reference-measure decomposition of
Proposition~\ref{prop:gradient}.\footnote{Two implementation conventions matter: the
sampler-returned threshold carries a gradient to $\theta$ that must be
stop-gradiented, and the numerics default to a \emph{strict} no-clamp mode.
Appendix~\ref{app:implementation} records both, with the ancillary tests that pin
them.}

\paragraph{Exactness ledger.}
Four objects must be kept distinct. The Horvitz--Thompson identity in
Theorem~\ref{thm:unbiased} is exact; the elimination of the $\binom nK$ subset sum in
Theorem~\ref{thm:tractable} is exact; replacing the remaining integral by $Q$-node
quadrature produces $\JhatQ$, a differentiable numerical approximation whose finite-$Q$
value and gradient errors are certified only on the configurations in
\S\ref{sec:numerics}; and the optional defensive clamps or node dropping add a second,
explicit source of bias when they engage. Thus ``unbiased'' below refers to the ideal
$\Jhat$ and its exact surrogate. The public finite-$Q$ loss is an approximation to that
surrogate; a uniform finite-$Q$ value-and-gradient error bound remains open.

\paragraph{Practical recipe.}
Algorithm~1 summarizes one training step; strict-mode numerics and the
$\kappa$-detach convention are Appendix~\ref{app:implementation}'s.

\begin{figure}[!tbp]
\centering
\fbox{\begin{minipage}{0.94\linewidth}
\textbf{Algorithm~1: one training step of the rank-conditioned estimator.}
\begin{enumerate}\setlength\itemsep{1pt}
\item Run exact SBS / Gumbel-Top-$(n{+}1)$ at width $n{+}1$; retain the top-$n$
pool, the threshold item, and the priority threshold $\kappa=G_{(n+1)}$.
\item Detach the threshold, $\kappa\leftarrow\sgop(\kappa)$, so autograd implements
the fixed-threshold decomposition of Proposition~\ref{prop:gradient}
(Appendix~\ref{app:implementation}).
\item Evaluate the pool rewards and the $n{+}1$ normalized sequence
log-probabilities, keeping their computation graphs.
\item Form $\JhatQ$ by the reward-sorted dynamic program of \S\ref{sec:tractable}
at the $Q$ quadrature nodes, in strict numerical mode.
\item Form the pool-only sampler log-density $\log f$ of
Corollary~\ref{cor:poolonly}.
\item Backpropagate the surrogate $-\big(\JhatQ+\sgop(\JhatQ)\log f\big)$.
\end{enumerate}
The quoted $O(n\log n+nKQ)$ cost covers step~4 only: not SBS decoding, reward
evaluation, or backpropagation through the $n{+}1$ sequence log-probabilities.
\end{minipage}}
\end{figure}

\section{Where the correction is needed (and where it is not)}
\label{sec:where}

Two facts localize the problem; both are proved below
(Observation~\ref{obs:coupling} by a closed-form three-item instance) and verified
numerically in \S\ref{sec:numerics}.

\begin{observation}[independent multi-start needs no correction]
\label{obs:pomo}
Let $R_1,\dots,R_n$ be rewards of rollouts drawn \emph{independently} from possibly
distinct per-start policies $\pi_s(\cdot\,;\theta)$ (as in standard POMO training when
the per-start action draws are conditionally independent given the instance and
deterministic encoder), and take the
independent-start best-of-$K$ objective
\[
\bar J_K=\binom nK^{-1}\!\!\sum_{\substack{A\subseteq[n]\\|A|=K}}\!\!
\E\Big[\max_{s\in A}R_s\Big],\qquad R_s\sim\pi_s\ \text{independently}.
\]
The usual i.i.d.\ Max@$K$ gradient weights are an \emph{unbiased} estimator of
$\nabla_\theta\bar J_K$ for arbitrary per-start marginals; mutual independence, not
identical distribution, is the property required by the proof.
\end{observation}

POMO's multi-start \citep{kwon2020pomo} is a \emph{deterministic symmetry enumeration},
not a random WOR draw. Under the standard implementation assumed here---deterministic
shared encoder computation and independent per-start action randomness conditional on
the instance---the rollouts are mutually independent. That independence (not identical
distribution or exchangeability) is all the subset argument needs, even under
heterogeneous starts, so no WOR correction is required in that regime. Shared dropout,
common random numbers, or another stochastic shared mechanism would violate the premise
and must be analyzed separately. This delimits where the correction belongs: it is
\emph{genuine} sampler-level coupling, not distinct-start independence, that breaks the
i.i.d.\ weights.

\begin{proof}
With $a_1,\dots,a_n$ drawn independently from $\pi_s(\cdot\,;\theta)$, the joint law
factorizes, $\log f_\theta(a_{1:n})=\sum_s\log\pi_s(a_s;\theta)$. For a fixed subset
$A$, the score-function identity gives
$\nabla_\theta\E[\max_{s\in A}R_s]
=\E\big[(\max_{s\in A}R_s)\sum_{s\in A}\nabla_\theta\log\pi_s(a_s;\theta)\big]$
(scores of $s\notin A$ drop by independence and $\E[\nabla\log\pi_s]=0$). Averaging over
the $\binom nK$ subsets and swapping the sums, each rollout $s$ receives the coefficient
$w_s=\binom nK^{-1}\sum_{A\ni s}\max_{s'\in A}R_{s'}$. Sorting the rewards ascending,
$R_{(1)}\le\dots\le R_{(n)}$, and splitting the subsets $A\ni s_{(i)}$ by which rank is
their maximum,
\[
w_{(i)}
=\binom nK^{-1}\Big(\ind\{i\ge K\}\tbinom{i-1}{K-1}\,R_{(i)}
\;+\;\sum_{j=i+1}^{n}\tbinom{j-2}{K-2}\,R_{(j)}\Big):
\]
the first term counts the subsets in which $s_{(i)}$ \emph{is} the maximum, the second
the \emph{support} terms from subsets whose maximum is a higher-ranked rollout $(j)$
(choose the remaining $K{-}2$ members from the $j-2$ ranks below $j$ excluding $i$).
These $w_{(i)}$ are exactly the i.i.d.\ Max@$K$ \emph{gradient} weights (PKPO's
transformed rewards); note they are distinct from the order-statistic \emph{value}
coefficients $\binom{i-1}{K-1}\binom nK^{-1}$, which weight $R_{(i)}$ in the estimate of
$\bar J_K$ itself and contain no support terms. Nothing in this computation uses
identical marginals, only mutual independence (the leave-one-out baseline likewise
requires only independence), so the estimator is unbiased for $\nabla_\theta\bar J_K$ at
every degree of start heterogeneity.
\end{proof}

Note $\bar J_K$ is a \emph{different estimand} from the coupled objective $\JW$ of
\S\ref{sec:prelim}, where the $n$ independent draws become a single WOR draw; we keep
the two distinct. WOR coupling breaks the factorization: $\log f_\theta$ acquires the
renormalization (coupling-normalizer) terms, and dropping them is the bias of
Observation~\ref{obs:coupling}.

\begin{observation}[coupling breaks the weights: a closed-form counterexample]
\label{obs:coupling}
Under a coupled WOR sampler (sequential draw over actions without repeats), the same
i.i.d.\ weights are biased. The three-item instance below exhibits the bias in closed
form, $\E[\widehat g_{\mathrm{iid}}]=\tfrac45\,\nabla_\theta\JW$ exactly; numerically,
the relative $L_2$ error is $0.35$--$0.72$ on a small skewed-reward example, and
Figure~\ref{fig:bias} maps the full $(K/n,\text{heterogeneity})$ sweep (median
$\approx0.82$). Thus mutual independence in Observation~\ref{obs:pomo} cannot be
replaced by marginal correctness or exchangeability alone.
\end{observation}

\begin{proof}[Closed-form instance]
Take $M=3$ items under a flat softmax at the uniform point ($\theta=0$,
$p_i=\tfrac13$), rewards $R=(1,0,0)$, and one coupled PL draw of size $K=n=2$, so the
draw is the whole pool: the instance isolates the weight bias from sample reuse. Only
item~$1$ carries reward, so
$\JW=\Pr(1\in\text{draw})
=p_1\big[1+\tfrac{p_2}{1-p_2}+\tfrac{p_3}{1-p_3}\big]$, and differentiating through
the softmax at $\theta=0$ gives
\[
\nabla_\theta\JW=\big(\tfrac{5}{18},\,-\tfrac{5}{36},\,-\tfrac{5}{36}\big).
\]
At $n=K$ the iid-grad weights of Observation~\ref{obs:pomo}'s proof assign
\emph{every} drawn item the weight $\max_{i\in S}R_i$ (the only size-$K$ subset is the
whole draw), so the naive update is
$\widehat g_{\mathrm{iid}}=(\max_{i\in S}R_i)\sum_{i\in S}\nabla_\theta\log p_i$. With
$P_{\WOR}(S)=\tfrac13$ for each $2$-subset,
$\nabla_{\theta_a}\log p_i=\delta_{ai}-p_a$, and the subset $\{2,3\}$ carrying zero
reward,
\[
\E\big[\widehat g_{\mathrm{iid}}\big]
=\tfrac13\big(\nabla\log p_1+\nabla\log p_2\big)
+\tfrac13\big(\nabla\log p_1+\nabla\log p_3\big)
=\big(\tfrac{2}{9},\,-\tfrac{1}{9},\,-\tfrac{1}{9}\big)
=\tfrac45\,\nabla_\theta\JW .
\]
The $2\leftrightarrow3$ symmetry forces both vectors onto one axis, so the entire
error is magnitude: cosine exactly $1$, norm short by exactly $\tfrac15$, the
low-heterogeneity signature of Figure~\ref{fig:bias}(a,b). No constant baseline can
repair it, since the whole-draw score $\sum_{i\in S}\nabla_\theta\log p_i$ has zero
mean at this symmetric point; the shortfall is exactly the dropped renormalization
part of the coupled score,
$\E\big[(\max_{i\in S}R_i)\big(\nabla_\theta\log P_{\WOR}(S)
-\sum_{i\in S}\nabla_\theta\log p_i\big)\big]$ ($=\tfrac1{18}$ in the first
component), the term identified before the observation. The corrected
estimator~($\star$) stays unbiased here (Theorem~\ref{thm:unbiased} applies,
$M=3\ge n+1$); since $n<2K$, the instance sits in Remark~\ref{rmk:variance}'s
infinite-second-moment regime: it isolates bias, not variance. All displayed
constants are certified to float64 precision by exact enumeration in the ancillary
suite (\texttt{anc/}).
\end{proof}

\emph{Estimator naming.} Three distinct ``i.i.d.-weight'' objects appear in this note
and are not interchangeable: \textbf{iid-grad} denotes the per-sample gradient weights
$w_{(i)}$ of Observation~\ref{obs:pomo}'s proof (PKPO's transformed rewards, including
the support terms; the arm of the closed-form instance above); \textbf{iid-SubLOO}
denotes iid-grad with the subset-leave-one-out baseline subtracted (the
variance-reduced form; the ``naive'' arm of Observation~\ref{obs:coupling}'s numerical
range and of Figure~\ref{fig:bias}); and \textbf{shared-value score} denotes the
scalar surrogate $V_{\mathrm{iid}}$ (the order-statistic \emph{value} estimate)
multiplied by the whole-pool score (the ``naive'' arm of Table~\ref{tab:variance}; at
$n=K$ it coincides with iid-grad, as in the closed-form instance). All three are
unbiased under independent sampling and generally biased under nontrivial coupling
(with exceptions such as $K=1$, constant rewards, or special symmetries); they are not
the same estimator, so each figure/table names the one it uses.

Beyond the closed-form instance, Appendix~\ref{app:biasmap} reports an illustrative
$(K/n,\text{heterogeneity})$ sweep on one toy geometry ($M=8$, $n=6$;
Figure~\ref{fig:bias}), decomposing the naive weights' error into direction,
magnitude, and absolute-error panels: the naive direction stays aligned everywhere
(cosine $\ge0.926$) while the magnitude is substantially off across most of the plane
(relative $L_2$ median $\approx0.82$). It is an illustration on one reward/policy
geometry, not a general regime map, which is why it sits in an appendix; the
corrected estimator~($\star$) is unbiased everywhere on it.

\section{Unbiasedness (Theorem 1)}
\label{sec:unbiased}

The construction below is an application of \emph{rank-conditioned}
Horvitz--Thompson estimation for bottom-$k$ / priority samples
\citep{cohen2008bottomk,duffield2007priority}: condition on the observed
$(n{+}1)$-st priority (the threshold; hence the synonym \emph{threshold-conditioned})
and reweight by the conditional inclusion probabilities. We do not claim that
conditioning device as new. The contribution is the resulting estimator for the
Plackett--Luce best-of-$K$ set total and the Max-specific collapse in
Section~\ref{sec:tractable}.

\begin{theorem}[unbiasedness]
\label{thm:unbiased}
For $K\le n$ and $M\ge n+1$, the estimator~($\star$) is unbiased:
$\;\E_{(\pool,\kappa)}[\Jhat]=\JW$.
\end{theorem}

\begin{proof}
By linearity, a fixed $K$-subset $S$ contributes to~($\star$) iff $S\subseteq\pool$, so
\[
\E[\Jhat]=\sum_{|S|=K}P_{\WOR}(S)\,\Big(\max_{i\in S}R_i\Big)\cdot
\E\!\left[\frac{\ind\{S\subseteq\pool\}}{\prod_{i\in S}q_i(\kappa)}\right],
\]
and it suffices to prove the \textbf{crux identity}
$\E\big[\ind\{S\subseteq\pool\}/\prod_{i\in S}q_i(\kappa)\big]=1$.

\emph{Exponential-clock form of the sampler.} Write $T_i=e^{-G_i}$; then
$\Pr(T_i>t)=\Pr(G_i<-\log t)=\exp(-e^{\phi_i}t)=e^{-p_it}$, so
$T_i\sim\Expo(p_i)$ independently (continuous since $p_i>0$ under the standing
full-support finite-policy model of \S\ref{sec:prelim}, hence a.s.\ distinct), and the
Gumbel-Top-$n$ sampler becomes an
\textbf{exponential race} \citep{vieira2014gumbel}: top-$n$ by $G$ $=$ the $n$ smallest $T$,
$\tau:=e^{-\kappa}=T_{(n+1)}$, and $q_i(\kappa)=\Pr(T_i<\tau)$.

\emph{Conditioning on the background.} Fix $S$ and condition on the background clocks
$B=\{T_j:j\notin S\}$ with order statistics $b_1<\dots<b_{M-K}$ ($b_{n-K+1}$ exists
because $M\ge n+1$). We claim
$\{S\subseteq\pool\}=\{\max_{i\in S}T_i<b_{n-K+1}\}$, and \emph{on this event}
$\tau=b_{n-K+1}$:
\begin{itemize}
\item[($\Leftarrow$)] If all $K$ clocks of $S$ are below $b_{n-K+1}$, then exactly $n$
clocks lie below it, the $K$ from $S$ plus $b_1,\dots,b_{n-K}$, so
$b_{n-K+1}=T_{(n+1)}=\tau$ and all of $S$ is among the $n$ smallest, i.e.\ in the pool.
\item[($\Rightarrow$)] If $S\subseteq\pool$, the $n-K$ background members of the pool
are the smallest background clocks $b_1,\dots,b_{n-K}$ (the pool is the $n$ smallest
overall), so the smallest non-pool clock is $b_{n-K+1}=\tau$, and every $S$-clock,
being in the pool, is $<\tau=b_{n-K+1}$.
\end{itemize}

\emph{Cancellation.} Given $B$, the weight
$\big(\prod_{i\in S}(1-e^{-p_ib_{n-K+1}})\big)^{-1}$ is $B$-measurable, and the
$S$-clocks are independent of $B$ with $\Pr(T_i<b_{n-K+1})=1-e^{-p_ib_{n-K+1}}$. On
$\{S\subseteq\pool\}$ the threshold is exactly $b_{n-K+1}$, so
$q_i(\kappa)=1-e^{-p_ib_{n-K+1}}$ there (off the event the indicator is $0$). Hence
\[
\E\!\left[\frac{\ind\{S\subseteq\pool\}}{\prod_{i\in S}q_i(\kappa)}\;\middle|\;B\right]
=\frac{\Pr\!\big(\max_{i\in S}T_i<b_{n-K+1}\mid B\big)}
      {\prod_{i\in S}\big(1-e^{-p_ib_{n-K+1}}\big)}
=\frac{\prod_{i\in S}\big(1-e^{-p_ib_{n-K+1}}\big)}
      {\prod_{i\in S}\big(1-e^{-p_ib_{n-K+1}}\big)}=1 .
\]
Taking $\E_B$ gives the crux identity, and therefore $\E[\Jhat]=\JW$.
\end{proof}

\begin{remark}[Finite-variance regime: $n\ge 2K$]
\label{rmk:variance}
Theorem~\ref{thm:unbiased} makes $\Jhat$ \emph{unbiased} for every $K\le n$, but unbiasedness does not by itself give a stable, finite-variance Monte Carlo estimator. The sharp threshold below is proved for the \emph{objective} estimator's Horvitz--Thompson terms; Proposition~\ref{prop:gradmoment} extends the same $n\ge2K$ threshold to a sufficient condition for the \emph{full} surrogate gradient. Each \emph{nonzero} degree-$K$ Horvitz--Thompson term $W_S\max_S R$ has a finite second moment only when the without-replacement pool is at least twice the target set,
\[
n\ge 2K\qquad(\text{equivalently } K/n\le \tfrac12).
\]
Reusing the exponential-race conditioning of the proof, a fixed $K$-subset $S$ carries the Horvitz--Thompson weight $W_S=\ind\{S\subseteq\pool\}/\prod_{i\in S}q_i(\kappa)$. Conditioning on the background clocks $B=\{T_j:j\notin S\}$ (on $\{S\subseteq\pool\}$ the threshold is $\tau=e^{-\kappa}=b_{n-K+1}$, the $(n{-}K{+}1)$-th smallest background clock) and integrating out the $S$-clocks gives the exact identity
\[
\E\!\left[W_S^2\,\middle|\,B\right]=\frac{1}{\prod_{i\in S}\bigl(1-e^{-p_i\,b_{n-K+1}}\bigr)} .
\]
As the threshold $\tau=b_{n-K+1}\to0^+$, the near-exhaustive, high-priority corner where the drawn pool is barely larger than $K$, one has $1-e^{-p_i\tau}\sim p_i\tau$ while the $(n{-}K{+}1)$-th background order statistic has density $\propto\tau^{\,n-K}$, so
\[
\E[W_S^2]\;\sim\;\int_{0^+}\tau^{-K}\cdot\tau^{\,n-K}\,d\tau=\int_{0^+}\tau^{\,n-2K}\,d\tau,
\]
which is \textbf{finite iff $n\ge 2K$ and diverges for $K/n>\tfrac12$} (i.e.\ $n<2K$; the boundary $n=2K-1$ is already log-divergent). Consequently, in the complementary near-exhaustive corner highlighted in Figure~\ref{fig:bias} ($K/n\to1$), $\Jhat$ is unbiased but has \emph{infinite} variance: a Monte-Carlo average over draws is still consistent, yet its sample variance never stabilizes: it is dominated by rare draws with a very high threshold, on which some $q_i(\kappa)\to0$ blow up $1/\prod_{i\in S}q_i$.

Absent exact cancellation across subset terms (e.g.\ zero or perfectly anti-correlated rewards, which can leave the aggregate finite even when individual terms are not), the complete single-draw $\Jhat$ inherits this $n\ge2K$ threshold. This is the variance of the \emph{objective} estimator $\Jhat$; the \emph{gradient} estimator additionally carries the score factor $\nabla_\theta\log f$, and Proposition~\ref{prop:gradmoment} shows the aggregate gradient nonetheless keeps a finite second moment whenever $n\ge2K$.

Numerically, the exact Gumbel-Top-$n$ sampler confirms this boundary: finite $\E[W_S^2]\sim O(10\text{--}10^2)$ for $n\ge2K$ versus realized weights $1/\prod_{i\in S}q_i$ exceeding $10^{4}$ in the divergent $n<2K$ cells (simulation details and the $n\ge4K$ estimability caveat in Appendix~\ref{app:secondmoment}). Three responses are available, with different guarantees:
\begin{itemize}
\item \emph{Over-sample the pool} so that $n\ge 2K$: the sampler controls this directly
by drawing a larger top-$n$ (the extra cost is a single wider draw).
\item \emph{Clip or self-normalize} the Horvitz--Thompson weights
$1/\prod_{i\in S}q_i(\kappa)$: this bounds the realized update but is a different,
biased estimator; no uniform ``small bias'' guarantee is claimed here.
\item \emph{Rao--Blackwellize}: replace the single-draw Horvitz--Thompson estimator with
its conditional expectation over the unordered size-$n$ pool set, the Max@$K$ analogue
of the unordered-set estimator of \citet{kool2020unordered}, which never forms the
$1/\prod q$ weight. This is a promising future route, not a result of this note;
separate companion work studies its definition and analysis.
\end{itemize}
This diagnosis dovetails with Figure~\ref{fig:bias} (Appendix~\ref{app:biasmap}): the variance divergence sits in the \emph{same} near-exhaustive, skewed corner ($K/n\to1$) where the naive i.i.d.\ weights are most wrong, so the correction is most needed exactly where the estimator's own variance is worst. The proved response is to enforce $n\ge2K$, which simultaneously moves the operating point out of that corner. Proposition~\ref{prop:gradmoment} closes the sufficiency half of the corresponding question for the \emph{full gradient}---including both $\nabla_\theta\Jhat$ and the sampler-score factor; whether $n<2K$ can ever leave the aggregated gradient square-integrable beyond degenerate cancellations remains open.
\end{remark}

\begin{remark}
At a high level Theorem~\ref{thm:unbiased} is Horvitz--Thompson estimation on a
population whose sampling units are the $K$-subsets of the pool, and estimating
higher-order (subset-product) quantities under priority / bottom-$k$ sampling has
precedent in the sketching literature: rank-conditioned HT estimation for bottom-$k$
sketches \citep{cohen2008bottomk} and unbiased subgraph (edge-product) estimators under
graph priority sampling \citep{ahmed2017gps}. We therefore do not claim degree-$K$
threshold conditioning per se as new; our contribution is its \emph{instantiation} for
the Plackett--Luce best-of-$K$ set functional: (a) casting the $\JW$ target as a total
over $K$-subsets, (b) the threshold-conditioned factorization that makes the joint
size-$K$ inclusion correction computable for Gumbel-Top-$n$ pools, and (c) the
Max-specific $O(nK)$-per-node collapse of that subset total (\S\ref{sec:tractable}).
At $K=1$ it is the classical single-draw Horvitz--Thompson / priority-sampling
result \citep{duffield2007priority,kool2020unordered}; the conditioning argument covers
the full \textbf{degree-$K$} case. Duffield's priority-sampling
analysis covers only additive subset-sum functionals (degree-$1$ Horvitz--Thompson
estimates and their degree-$2$ pairwise covariances). Mechanistically, our estimator is the coupled
without-replacement analogue of the i.i.d.\ ``probability-of-being-the-maximum''
$U$-statistic of PKPO / RSPO / OrderGrad
\citep{pkpo2025,rspo2025,ordergrad2026}, with unbiasedness restored not by
independence but by inclusion-probability (Horvitz--Thompson) reweighting; RSPO's
closed-form i.i.d.\ probability-of-being-the-max \citep{rspo2025} is the closest
mechanism, and it holds only for independent draws. The gradient inherits this
unbiasedness: the score-function surrogate gradient
$\widehat g=-\nabla_\theta\mathcal L$ satisfies $\E[\widehat g]=\nabla_\theta\JW$,
\emph{proved} as Proposition~\ref{prop:gradient} below (and independently confirmed
against exact enumeration to quadrature tolerance, \S\ref{sec:numerics}).
\end{remark}

\begin{proposition}[unbiased Max@$K$ policy gradient]
\label{prop:gradient}
Adopt the finite-support regime of Theorem~\ref{thm:unbiased}: a finite item set
$\{1,\dots,M\}$ with $M\ge n+1$, a full-support normalized policy
$p_i=p_\theta(i)>0$ that is continuously differentiable in a finite-dimensional
parameter $\theta$, fixed bounded rewards $|R_i|\le R_{\max}$, and $1\le K\le n$.
Let $f_\theta$ be the joint law of the
Gumbel-Top-$n$ draw ($n\ge K$; $K$ enters only through $\Jhat$): its pool $\pool$, its
threshold item, and its priority $\kappa=G_{(n+1)}$, and let
\[
\widehat g\;:=\;\nabla_\theta\Jhat(\pool,\kappa)\;+\;\sgop\!\big(\Jhat(\pool,\kappa)\big)\,\nabla_\theta\log f_\theta
\]
be the score-function surrogate gradient of \S\ref{sec:prelim} (equivalently
$-\nabla_\theta\mathcal L$ for $\mathcal L=-(\Jhat+\sgop(\Jhat)\log f)$: the pathwise
term through $P_{\WOR}$ and $q_i$ plus the sampler score). Then $\widehat g$ is an
unbiased estimator of the true Max@$K$ policy gradient:
\[
\E_{(\pool,\kappa)\sim f_\theta}\!\big[\widehat g\big]\;=\;\nabla_\theta\JW .
\]
\end{proposition}

\begin{proof}[Proof idea]
$\Phi(\theta):=\E_{f_\theta}[\Jhat]\equiv\JW$ at every full-support $\theta$
(Theorem~\ref{thm:unbiased}), so $\nabla_\theta\Phi=\nabla_\theta\JW$, and the claim
reduces to differentiating under the sampler expectation. The support of $f_\theta$ is
$\theta$-free, the discrete $(P,m)$-part interchanges termwise, and the remaining 1-D
$\tau$-integral admits an integrable dominating envelope
$C(1+\tau)e^{-(M-n)p_{\min}\tau}$, uniform on a compact neighborhood of $\theta$: the
Horvitz--Thompson weight's $\tau\to0$ blow-up is annihilated by the pool factor of
$f_\theta$, the same cancellation as in the crux identity. Appendix~\ref{app:gradproof}
gives the full argument.
\end{proof}

Unbiasedness alone does not certify trainability. The next result shows that the
$n\ge2K$ regime of Remark~\ref{rmk:variance} controls not only the objective terms
but the \emph{entire} surrogate gradient, score factor included.

\begin{proposition}[full-gradient second moment: $n\ge2K$ suffices]
\label{prop:gradmoment}
Adopt the assumptions and notation of Proposition~\ref{prop:gradient}, fix $\theta$,
and set $L:=\max_{i\le M}\|\nabla_\theta\log p_i(\theta)\|<\infty$. If $n\ge2K$, the
surrogate gradient is square-integrable:
\[
\E_{(\pool,\kappa)\sim f_\theta}\big\|\widehat g\big\|^2
\;\le\;
2R_{\max}^2L^2\,\max_{|S|=K}\,
\E\Big[\big(9K^2+(n{+}1{+}\tau)^2\big)\,W_S^2\Big]\;<\;\infty,
\]
where $\tau=e^{-\kappa}$ and
$W_S=\ind\{S\subseteq\pool\}/\prod_{i\in S}q_i(\kappa)$ is the degree-$K$
Horvitz--Thompson weight of Remark~\ref{rmk:variance}. In particular $n\ge2K$ is a
sufficient training-stability condition for the full gradient estimator, not only
for its objective terms. We claim sufficiency only: term-level divergence for
$n<2K$ is Remark~\ref{rmk:variance}'s, and whether exact cross-term cancellations
can leave the aggregated gradient square-integrable for $n<2K$ on special instances
is open.
\end{proposition}

\begin{proof}[Proof idea]
Three elementary score bounds hold pointwise, \emph{uniformly in the threshold}.
(i)~The inclusion scores never blow up:
$\nabla_\theta\log q_i(\kappa)=\varphi(p_i\tau)\,\nabla_\theta\log p_i$ with
$\varphi(x)=x/(e^x-1)\in(0,1)$ for $x>0$, so $\|\nabla_\theta\log q_i\|\le L$ even
as $q_i\to0$: the $1/q_i$ blow-up affects the weight, never its score.
(ii)~$\|\nabla_\theta\log P_{\WOR}(S)\|\le2KL$, because every Plackett--Luce
denominator is the mass of the remaining items, whose log-gradient is a convex
combination of the $\nabla_\theta\log p_j$. (iii)~The sampler score grows at most
linearly in the clock threshold, $\|\nabla_\theta\log f_\theta\|\le L(n{+}1{+}\tau)$.
Together these give the master pointwise bound
\[
\|\widehat g\|\;\le\;R_{\max}L\,\big(3K+n+1+\tau\big)
\sum_{|S|=K}P_{\WOR}(S)\,W_S ,
\]
and since $\sum_{|S|=K}P_{\WOR}(S)=1$, Jensen's inequality reduces
$\E\|\widehat g\|^2$ to the displayed per-term moments. Those are finite for
$n\ge2K$ by the exponential-race conditioning of Remark~\ref{rmk:variance}: the
extra $(n{+}1{+}\tau)^2$ factor is harmless, bounded near $\tau=0$ where the weight
blows up and polynomial against the threshold's exponential upper tail.
Appendix~\ref{app:gradmoment} gives the full proof; the three score bounds and the
master inequality are verified pointwise by autodiff on randomized draws, including
extreme thresholds, in the ancillary suite
(\texttt{anc/tests/test\_gradient\_moment.py}).
\end{proof}

\section{An exact combinatorial collapse (Theorem 2)}
\label{sec:tractable}

($\star$) is a sum over $\binom nK$ subsets, but a brute-force evaluation is far more
costly: forming each $P_{\WOR}(S)$ from its $K!$ orderings makes the total work
$\binom nK K!$, $\approx5\times10^8$ at $n{=}16,K{=}8$. The subset sum nonetheless
collapses exactly. The engine is an integral representation of $P_{\WOR}$ whose integrand
factorizes over items.

\begin{lemma}[exponential-race representation of $P_{\WOR}$]
\label{lem:pwor}
With $\sum_i p_i=1$ and $P_S=\sum_{i\in S}p_i$,
\[
P_{\WOR}(S)=\int_0^\infty e^{-t}\,\sum_{i\in S} p_i
\prod_{i'\in S\setminus\{i\}}\!\big(e^{p_{i'}t}-1\big)\,dt .
\]
\end{lemma}

\begin{proof}
Attach independent $X_i\sim\Expo(p_i)$. Their ascending order is a PL sequential draw
with weights $p$, so a size-$|S|$ WOR draw yields $S$ iff the $|S|$ smallest $X$'s are
precisely $S$:
\[
P_{\WOR}(S)=\Pr\!\Big(\max_{i\in S}X_i<\min_{j\notin S}X_j\Big)
=\int_0^\infty \frac{d}{dt}\!\prod_{i\in S}\big(1-e^{-p_i t}\big)\;
\prod_{j\notin S} e^{-p_j t}\,dt .
\]
The derivative is $\sum_{i\in S} p_i e^{-p_i t}\prod_{i'\ne i}(1-e^{-p_{i'}t})$ and
$\prod_{j\notin S}e^{-p_j t}=e^{-(1-P_S)t}=e^{-t}\prod_{i\in S}e^{p_i t}$.
Distributing $\prod_{i\in S}e^{p_i t}$ collapses each term via
$e^{p_it}p_ie^{-p_it}=p_i$ and $e^{p_{i'}t}(1-e^{-p_{i'}t})=e^{p_{i'}t}-1$, giving the
integrand $e^{-t}\sum_{i\in S}p_i\prod_{i'\ne i}(e^{p_{i'}t}-1)$.
\end{proof}

The integrand \textbf{factorizes over items}, which is what collapses the subset sum in
($\star$). Define $h_i(t)=(e^{p_it}-1)/q_i(\kappa)$ and $c_i(t)=p_i/(e^{p_it}-1)$, so
$p_i/q_i=c_ih_i$.

\begin{theorem}[exact subset-sum collapse; numerical integral]
\label{thm:tractable}
$\displaystyle\Jhat=\int_0^\infty e^{-t}\,F(t)\,dt$, where
\[
F(t)=\sum_{\substack{S\subseteq\pool\\|S|=K}}
\Big(\max_{i\in S}R_i\Big)\Big(\prod_{i\in S}h_i(t)\Big)\Big(\sum_{i\in S}c_i(t)\Big),
\]
and $F(t)$ is computable in $O(nK)$. The combinatorial subset sum thus collapses
\emph{exactly} to a single one-dimensional integral. If $(t_r,w_r)_{r=1}^Q$ are
Gauss--Laguerre nodes and weights, define the practical approximation
\[
\JhatQ:=\sum_{r=1}^Q w_r F(t_r).
\]
It costs $O(n\log n+nKQ)$ arithmetic operations in the pool size (the $n\log n$ is
the reward sort), after exact elimination of the $\binom nK$-term subset sum, plus a
one-time $O(M)$ softmax precompute for a flat parameterization. Finite $Q$ introduces
value and gradient quadrature error: we do not certify an explicit
$Q(\epsilon,p_{\min},k,n)$ bound, so this is a fixed-node arithmetic cost, not a
polynomial-time $\epsilon$-approximation guarantee and not an exact finite-$Q$
unbiasedness claim.
\end{theorem}

\begin{proof}[Proof sketch]
Substitute Lemma~\ref{lem:pwor} into ($\star$) and exchange the finite sum with the
integral (Fubini). Using $p_i/q_i=c_ih_i$ and $(e^{p_{i'}t}-1)/q_{i'}=h_{i'}$, the
inner factor rewrites
\[
\frac1{\prod_{i\in S}q_i}\sum_{i\in S}p_i\!\!\prod_{i'\ne i}\!(e^{p_{i'}t}-1)
=\sum_{i\in S}c_i h_i\!\!\prod_{i'\ne i}\!h_{i'}
=\sum_{i\in S}c_i\prod_{i''\in S}h_{i''}
=\Big(\prod_{i\in S}h_i\Big)\Big(\sum_{i\in S}c_i\Big),
\]
which is $F(t)$ and proves the exact integral identity.

\emph{Sort-by-reward collapse.} Order the pool $R_1\ge\dots\ge R_n$, breaking reward
ties by a fixed item order. Assign each subset to its lowest-index maximizer, a unique
representative even when several items share the maximum reward. Thus partition
$S$ by that representative: $S=\{j\}\cup T$ with
$T\subseteq\{j{+}1,\dots,n\}$, $|T|=K-1$, $\max_S R=R_j$. Then
\[
F(t)=\sum_{j=1}^{n} R_j\, h_j\big(c_j A_j+B_j\big),\quad
A_j=\!\!\sum_{\substack{T\subseteq\{j+1..n\}\\|T|=K-1}}\!\prod_{i\in T}h_i,\quad
B_j=\!\!\sum_{\substack{T\subseteq\{j+1..n\}\\|T|=K-1}}\!
\Big(\prod_{i\in T}h_i\Big)\Big(\sum_{i\in T}c_i\Big),
\]
using $\big(\prod_S h\big)\big(\sum_S c\big)=h_j\prod_T h\,\big(c_j+\sum_T c\big)$. The
dynamic program below produces all $A_j,B_j$ in $O(nK)$; Gauss--Laguerre approximates
the $1$-D integral. To verify that the ideal integral is finite, write
$p_{\min,S}=\min_{i\in S}p_i>0$, each subset's contribution behaves as
$e^{-t}\big(\prod_S h\big)\big(\sum_S c\big)\sim(p_{\min,S}/\!\prod_{i\in S}q_i)\,
e^{-(1-P_S+p_{\min,S})t}$ (the $c_i(t)=p_i/(e^{p_it}-1)$ factor decays like
$p_ie^{-p_it}$), so with a full-support policy ($p_{\min,S}>0$) the integrand decays at
a strictly positive rate $\ge\min_S(1-P_S+p_{\min,S})>0$ for \emph{all}
$1\le K\le n\le M$, including the boundary $P_S=1$. This establishes integrability
of the ideal expression. It does not by itself prove monotonic Gauss--Laguerre error
decay or a uniform finite-$Q$ rate; convergence under node refinement is checked on
the finite configurations in \S\ref{sec:numerics}, and a certified error bound in
terms of $(p_{\min},k,n)$ is left to future work. (A coarser decay bound is
$O\!\big(e^{-(1-P_{\max})t}\big)$ with
$P_{\max}=\max_{|S|=K}\sum_{i\in S}p_i$.)
\end{proof}

\paragraph{The dynamic program.}
Process items $j=n,\dots,1$ (increasing reward). Maintain, over the current
\textbf{below-set} $\{j{+}1,\dots,n\}$, for $m=0,\dots,K$,
\[
E[m]=\!\!\sum_{|T|=m}\prod_{i\in T}h_i,\qquad
G[m]=\!\!\sum_{|T|=m}\Big(\prod_{i\in T}h_i\Big)\Big(\sum_{i\in T}c_i\Big),
\]
initialized $E[0]=1,\,G[0]=0$. At step $j$: read $A_j=E[K{-}1]$, $B_j=G[K{-}1]$,
accumulate $R_j h_j(c_jA_j+B_j)$ into $F$, then \textbf{insert} item $j$ (descending
$m$ so the right-hand sides stay pre-insert),
\[
G[m]\leftarrow G[m]+h_j\,G[m{-}1]+c_j h_j\,E[m{-}1],\qquad
E[m]\leftarrow E[m]+h_j\,E[m{-}1]\qquad(m=K,\dots,1).
\]
\emph{Correctness of the insert.} A size-$m$ subset of the enlarged below-set either
omits $j$ (old $E[m],G[m]$) or contains $j$: then it is $\{j\}\cup T'$ with
$|T'|=m{-}1$, the product gains $h_j$ and the $c$-sum gains $c_j$, so the included case
sums to $h_j(G[m{-}1]+c_jE[m{-}1])$. Each insert is $O(k)$; $n$ inserts give $O(nK)$
per node, hence $O(nKQ)$ overall (plus the $O(n\log n)$ reward sort, and, for a flat
softmax parameterization only, a one-time $O(M)$ normalizer precompute).

\begin{corollary}[pool-only form]
\label{cor:poolonly}
Every quantity in~($\star$), in the DP above, and in the score-function surrogate of
\S\ref{sec:prelim} is computable from the drawn items' probabilities alone together with
the normalization $\sum_y p_\theta(y)=1$ (and, for gradients, from those
probabilities' differentiable computation graphs): the inclusion probabilities
$q_i(\kappa)$~\eqref{eq:incl} and the factors $h_i,c_i$ involve only
$\{p_i:i\in\pool\}$; $P_{\WOR}(S)$ for $S\subseteq\pool$ uses only pool probabilities and
$\sum_i p_i=1$ (Lemma~\ref{lem:pwor}); and the sampler log-density collapses, in the
exponential-clock coordinates $\tau=e^{-\kappa}$ of Theorem~\ref{thm:unbiased}'s proof
(joint density as in Appendix~\ref{app:gradproof}), to
\[
\log f_\theta(\pool,m,\tau)
=\sum_{i\in\pool}\log q_i(\kappa)\;+\;\log p_m-p_m\tau\;-\;\tau\Big(1-\textstyle\sum_{i\in\pool}p_i-p_m\Big),
\]
computable from $\{p_\theta(y_i):i\in\pool\cup\{m\}\}$.
\end{corollary}

\begin{proof}
Only the far-item factor of $f_\theta$ involves items outside the draw, and it
telescopes through the total mass: $\log(1-q_j)=-e^{-\kappa}p_j$ \emph{exactly} (since
$1-q_j=\exp(-p_j\tau)$), so
$\sum_{j\notin\pool\cup\{m\}}\log(1-q_j)=-\tau\,(1-\sum_{i\in\pool}p_i-p_m)$.
\end{proof}

Hence no enumeration of the support is required: an exact-SBS decoder that reports its
$n{+}1$ sampled sequences' log-probabilities and the priority threshold suffices to form
the ideal integrand, its finite-$Q$ approximation, and the associated surrogate gradient
(whose pathwise term then flows into the drawn sequences' log-probabilities, i.e.\ back
through the decoder). The reference
implementation in the ancillary files ships these pool-only forms alongside the
flat-softmax ones and checks them equal to float64 precision. The sequence-scale
case deserves its own formal statement:

\begin{corollary}[finite structured sequence policies]
\label{cor:sequence}
Let the items be the terminal sequences of a finite decoding tree (fixed-length or
length-capped), with normalized full-support terminal probabilities
$p_\theta(y)>0$, $\sum_yp_\theta(y)=1$, each continuously differentiable in
$\theta$; let rewards be bounded and $\theta$-independent; and let exact Stochastic
Beam Search at width $n{+}1$ return the top-$(n{+}1)$ Gumbel-perturbed terminal
sequences and the priority threshold $\kappa=G_{(n+1)}$. Then
Theorem~\ref{thm:unbiased}, Propositions~\ref{prop:gradient}
and~\ref{prop:gradmoment}, and Theorem~\ref{thm:tractable} apply verbatim with
``items'' $=$ terminal sequences, and every required quantity is a function of the
values and differentiable computation graphs of the $n{+}1$ drawn sequences'
log-probabilities, plus their rewards (Corollary~\ref{cor:poolonly}): no
enumeration of the sequence support is needed.
\end{corollary}

\begin{proof}
The terminal sequences form a fixed, finite, full-support, normalized item set, so
the finite-support hypotheses of the cited results hold verbatim. Exact SBS with
top-down Gumbel propagation draws precisely the top-$(n{+}1)$ of independently
Gumbel-perturbed terminal log-probabilities \citep{kool2019sbs}, i.e.\ the
Gumbel-Top-$(n{+}1)$ draw of \S\ref{sec:prelim} on that item set, including its
threshold. Corollary~\ref{cor:poolonly} gives the pool-only computability. The
ancillary suite certifies the sampler law and the pool-only/flat equality on an
enumerable structured tree (\texttt{anc/tests/test\_sbs.py}).
\end{proof}

\paragraph{Positioning of the machinery.}
We do \emph{not} claim the integral representation of Lemma~\ref{lem:pwor} or its
numerical evaluation as novel: \citet{kool2020unordered} already give the same
Gumbel / exponential-race integral for without-replacement inclusion probabilities
(for the \emph{mean} objective) and observe that it can be integrated numerically, and
the elementary-symmetric ``insert one item at a time'' collapse is standard
conditional-Poisson / top-$k$ machinery
\citep{chen1994conditional,meister2021cpsbs,oosterhuis2021plrank}. The load-bearing
novelty of
Theorem~\ref{thm:tractable} is that prior WOR estimators
\citep{kool2020unordered,duffield2007priority,cohen2009varopt} collapse only
\emph{mean / additive} Horvitz--Thompson functionals, a linear sum over sampled
items, whereas here we collapse a genuine sum \emph{over the $\binom nK$ subsets}, each
carrying a \textbf{max} weight and a $P_{\WOR}(S)$ that itself costs $K!$ to form
naively (brute-force total $\binom nK K!$), via the reward-sorted telescoping plus
the exponential-race integral. Concretely, we \emph{assemble} Kool's integral, an elementary-symmetric DP,
and Gauss--Laguerre quadrature into the first $O(n\log n + nKQ)$ collapsed evaluation we
are aware of for this sample-reuse Max@$K$ WOR estimator, improving the exact-but-$O(2^K)$
inclusion--exclusion evaluation
that Kool's machinery would give when applied term-by-term to ($\star$) (and, a
fortiori, the naive $\binom{n}{K}K!$ enumeration).

\begin{proposition}[reductions]
\label{prop:reductions}
At its two boundaries the estimator~($\star$) connects to known estimators: an
\emph{exact} reduction at $K=1$, and a \emph{weight-profile} limit for PKPO (equal only
up to a positive scalar).
\begin{itemize}
\item[(i)] \emph{$K=1$.} $A_j=1,\,B_j=0$, so $F(t)=\sum_j R_j h_j c_j=\sum_j R_j p_j/q_j$
(independent of $t$), giving $\Jhat=\sum_{i\in\pool}(p_i/q_i)R_i$, the
Horvitz--Thompson single-draw estimator of \citet{duffield2007priority,kool2020unordered}.
\item[(ii)] \emph{Uniform-policy (exchangeable) limit.} In the uniform-policy limit
$p_i\to 1/M$ the induced size-$K$ set law becomes
\emph{exchangeable}, $P_{\WOR}(S)\to\binom MK^{-1}$, uniform over subsets, the
without-replacement analogue of an exchangeable pool (the WOR negative correlation
persists, distinct draws still cannot repeat; only the heterogeneity vanishes),
\emph{and simultaneously} the
inclusion probabilities equalize, $q_i(\kappa)\to q(\kappa)$. Hence every retained subset
carries one common Horvitz--Thompson coefficient, and the coefficients that $\Jhat$
assigns to the sorted pool rewards match PKPO's order-statistic weight profile
$\binom{i-1}{K-1}\binom nK^{-1}$ \citep{pkpo2025} up to that common positive scalar.
This is a profile correspondence, not an estimator identity: self-normalization is not
part of~($\star$) and generally sacrifices its exact unbiasedness.
\end{itemize}
\end{proposition}

\begin{proof}
(i) is immediate from the $K{=}1$ specialization of the DP ($E[0]{=}1$, $G[0]{=}0$).

(ii) In the uniform-policy limit $p_i\to 1/M$ we have $\phi_i\to-\log M$ for every $i$, so
by~\eqref{eq:incl} the inclusion probabilities equalize,
\[
q_i(\kappa)\;\longrightarrow\;q(\kappa)=1-\exp\!\big(-e^{-\kappa}/M\big),
\]
independent of $i$. Simultaneously, a Plackett--Luce draw with equal weights makes every
ordered $K$-tuple of distinct items equiprobable, so each size-$K$ \emph{set} has
\[
P_{\WOR}(S)=K!\,\frac{(M-K)!}{M!}=\binom MK^{-1},
\]
uniform over subsets, the exchangeable, without-replacement analogue of an i.i.d.\
pool. It is this \emph{joint} degeneration ($P_{\WOR}$ becoming uniform \emph{and} the
$q_i$ equalizing together), not the equal-inclusion condition alone, that makes the
per-subset weight in~($\star$) collapse to a single constant, the same for \emph{every}
retained subset $S\subseteq\pool$:
\[
\frac{P_{\WOR}(S)}{\prod_{i\in S}q_i(\kappa)}\;\longrightarrow\;
w(\kappa):=\frac1{\binom MK\,q(\kappa)^{K}} .
\]
Therefore
\[
\Jhat\;\longrightarrow\;w(\kappa)\!\!\sum_{\substack{S\subseteq\pool\\|S|=K}}\!\!\max_{i\in S}R_i
\;=\;\Big(w(\kappa)\tbinom nK\Big)\cdot
\underbrace{\binom nK^{-1}\!\!\sum_{\substack{S\subseteq\pool\\|S|=K}}\!\!\max_{i\in S}R_i}_{\text{PKPO order-statistic average over the pool}} .
\]
Sorting the pool $R_{(1)}\le\dots\le R_{(n)}$, item $(i)$ is the maximizer of exactly
$\binom{i-1}{K-1}$ of the $\binom nK$ subsets, so the bracketed average equals
$\binom nK^{-1}\sum_i\binom{i-1}{K-1}R_{(i)}$, precisely PKPO's order-statistic weights
$\binom{i-1}{K-1}\binom nK^{-1}$ \citep{pkpo2025}. This is exactly the order-statistic
profile of claim~(ii): in the uniform-policy limit $\Jhat$ assigns PKPO's weights to the
sorted pool rewards, up to the single common scalar $w(\kappa)\binom nK$.
\end{proof}

\section{Numerical certificates on finite models}
\label{sec:numerics}

\begin{itemize}
\item \textbf{Unbiasedness (Thm.~\ref{thm:unbiased}) and gradient
(Prop.~\ref{prop:gradient}).} We validate $\E[\Jhat]=\JW$ and
$\E[\widehat g]=\nabla_\theta\JW$ against the exact-enumeration ground truth by a
\emph{deterministic} Gauss--Legendre quadrature over the priority $\kappa$ (not a
Monte-Carlo average). The shipped fine-grid certificate uses $M=5$, $n=3$,
$K\in\{2,3\}$, logits $(0.3,-0.2,0.1,-0.1,0.4)$, and rewards
$(0,1,2,4,10)$. At a $32{\times}16$-panel $\kappa$-grid, the value error, gradient
relative $L_2$, and sampler-mass error are each below $10^{-11}$; refining the coarse
$16{\times}10$ grid reduces the residual by at least two orders of magnitude. This
certificate checks the theorem implementation but is not its proof. Dropping the
sampler-score term produces relative gradient error above $10\%$ on the declared cell,
confirming that term is essential. Separately, the degree-$1$ Horvitz--Thompson total
$\E[\sum_{i\in\pool}a_i/q_i(\kappa)]=\sum_i a_i$ is confirmed under the \emph{actual}
Gumbel-Top-$n$ sampler by a $4\times10^{5}$-draw Monte-Carlo at $(M,n)=(6,3)$ to
$\pm3\%$, and the estimator-weight second moment is verified finite for $n\ge2K$ and
divergent for $n<2K$ (Remark~\ref{rmk:variance}). The three pointwise score bounds
and the master gradient inequality behind Proposition~\ref{prop:gradmoment} are
verified by autodiff on randomized draws and on threshold grids far outside the
typical sampling range.\footnote{The value and gradient
identity (coarse- and fine-grid), poly-vs-brute agreement on a broader configuration
and reward grid, the second-moment diagnostic, the gradient-moment pointwise
bounds, the variance comparison
(Table~\ref{tab:variance}), the observations of \S\ref{sec:where}, and
Figure~\ref{fig:bias} are reproduced by the validation suite shipped as
\emph{ancillary files} with this submission: from the \texttt{anc/} directory, run
\texttt{python3 -m pytest tests} (see \texttt{anc/README.md} for the per-claim test map
and the Table~\ref{tab:variance} / Figure~\ref{fig:bias} scripts). Wall-clock figures
are single-CPU-core measurements and are hardware-dependent.}
\item \textbf{Exact collapse (Thm.~\ref{thm:tractable}).} The $O(nKQ)$ form matches the
brute-force sum to machine precision ($\lesssim3\times10^{-13}$ across the validated
grid; Table~\ref{tab:tractable}). At the feasibility boundary $(10,5)$ it runs roughly
$650\times$ faster than brute force (median $701$~ms vs.\ $1.1$~ms over repeated
single-core runs; hardware-dependent), and stays in the millisecond range where brute
force is infeasible.
\item \textbf{Variance diagnostic vs.\ generic REINFORCE
(Appendix~\ref{app:variance}).} A deliberately small finite-model diagnostic (one
random reward/logit geometry family, $K=2$ only; a sanity check, not a regime study)
confirms the qualitative picture at equal reward-evaluation budget: the shared-value
score surrogate is \emph{biased} throughout (median relative bias $0.6$--$1.2$);
joint-score REINFORCE and the ideal HT estimator are both unbiased (the experiment uses
the certified finite-$Q$ implementation); the HT arm's
cost-normalized variance is lower once the pool is oversampled ($K/n\le\tfrac13$) and
is worse and heavy-tailed at its $K/n{=}\tfrac12$ finite-variance boundary, exactly
as Remark~\ref{rmk:variance} predicts. Two caveats are integral to reading it: the
cost normalization counts \emph{reward evaluations}, the right unit only when reward
evaluation dominates the step cost, and per-draw wall-clock is $\approx10\times$
higher for the HT arm at these sizes (the DP). Protocol, full table, and tail
diagnostics are in Appendix~\ref{app:variance} (Table~\ref{tab:variance}). The HT
form's primary value remains its exact, reusable structure and its role as the base
for the companion paper's Rao--Blackwellization; its variance case is
regime-dependent, not uniform.
\end{itemize}

\begin{table}[!tbp]
\centering
\caption{Collapsed estimator (Thm.~\ref{thm:tractable}) vs.\ the brute-force
$\binom nK K!$ sum. Relative error is machine-precision; wall clocks are medians [IQR]
over $15$ warm-started single-core repetitions (Apple~M4, float64, $Q=96$;
hardware-dependent). ``---'' marks configurations where brute force is infeasible
(brute force at $(10,5)$: median $701$~ms, IQR $[655,718]$).}
\label{tab:tractable}
\begin{tabular}{lccc}
\toprule
$(n,K)$ & rel.\ error & brute-force terms & collapsed wall (ms) \\
\midrule
$(3,2)\dots(10,5)$ & $\lesssim 3\times10^{-13}$ & $6\dots 3\times10^{4}$ & $\le1.07$ $[0.89,1.82]$ \\
$(16,8)$ & --- (brute infeasible) & $5.2\times10^{8}$ & $1.62$ $[0.91,1.99]$ \\
$(50,10)$ & --- & $3.7\times10^{16}$ & $3.91$ $[3.36,4.71]$ \\
\bottomrule
\end{tabular}
\end{table}

\section{Related work}
\label{sec:related}

Three contrasts position the note; a closing paragraph supplies application-domain
context.

\paragraph{(1) i.i.d.\ Max@$K$ sample reuse.}
The unbiased \textbf{i.i.d.} Max@$K$ estimator is not ours: PKPO \citep{pkpo2025}
derives the order-statistic weights $\binom{i-1}{K-1}\binom nK^{-1}$, the unbiased
gradient, $K$-train-vs-$K$-eval tunability, and $K$-annealing; RSPO \citep{rspo2025}
the risk-seeking variant (in the lineage of risk-seeking policy gradients,
\citealp{petersen2021dsr}); MaxPO \citep{maxpo2026} the finite-batch advantage and a
Leave-Two-Out baseline, explicitly deferring ``settings with correlated
generations''; OrderGrad \citep{ordergrad2026} unbiased gradients for all order
statistics; BoNW \citep{bonw2025} an on-policy max@$K$ gradient with an off-policy
extension; Advantage Shaping \citep{advshaping2026} the unifying surrogate-reward
view. All of these draw $n$ responses i.i.d.\ (with replacement) and only then
sub-sample size-$K$ subsets of the \emph{realized} batch without replacement: a
$U$-statistic device targeting $\JIID$, not a change of sampler. Our estimand $\JW$
and estimator sit outside that setup.

\paragraph{(2) Generic WOR policy gradients and priority sampling (mean objective).}
Unbiased gradients for exactly samplable WOR draws \emph{per se} are classical:
joint-score REINFORCE on the Plackett--Luce draw, variance-reduced by
\citet{gadetsky2020plackett}; \citet{kool2019sbs,kool2020unordered} supply the
Gumbel-Top-$k$ / SBS sampler, WOR-trained REINFORCE on TSP, and the unbiased WOR
estimator for the \textbf{mean}. The priority-sampling / rank-conditioning lineage
\citep{duffield2007priority,cohen2009varopt,cohen2008bottomk} estimates additive WOR
subset sums, and graph priority sampling forms unbiased \emph{products} of edge
estimators for subgraph counts \citep{ahmed2017gps}: the higher-order
Horvitz--Thompson precedent we credit directly. A meta-item reduction to
\citet{kool2020unordered} exists in principle (take the $\binom MK$ subsets as
categorical items with $p(x)=P_{\WOR}(S)$ and $f(x)=\max_{i\in S}R_i$), but it
requires sampling subset-valued outcomes and evaluating their probabilities, and it
neither reuses the $\binom nK$ subsets embedded in one item-level pool nor supplies
the $O(nK)$-per-node collapse; we reuse Kool et al.'s exponential-race derivation
technique, not their general arbitrary-function principle. A complementary line
relaxes the subset draw to get pathwise gradients \citep{xie2019subsets,paulus2020sst}:
biased for the discrete objective (vanishing only in the zero-temperature limit) and
targeting no Max@$K$ sample reuse, so the two lines trade bias for variance and
neither subsumes the other.

\paragraph{(3) The delta of this note.}
Item-level pool reuse plus a Max-specific collapse: a sample-reuse ($n>K$) estimator
of $\JW$ for threshold-observed Gumbel-Top-$n$ / Plackett--Luce sampling,
gradient-side, \textbf{unbiased} (Theorem~\ref{thm:unbiased},
Proposition~\ref{prop:gradient}), with an exact subset-sum collapse and fixed-$Q$
$O(n\log n+nKQ)$ evaluation (Theorem~\ref{thm:tractable}), a certified $n\ge2K$
second-moment regime for both the objective terms and the full gradient
(Remark~\ref{rmk:variance}, Proposition~\ref{prop:gradmoment}), reducing to
Kool / Duffield at $K=1$ (an exact estimator identity) and to PKPO's order-statistic
weight profile (up to a positive scalar) in the uniform-policy limit (both proved,
Proposition~\ref{prop:reductions}). Generic joint-score REINFORCE already covers
WOR~$\times$~max; the narrow contribution is the rank-conditioned reuse of a larger
pool and the exact Max-specific collapse.

\paragraph{Application-domain context.}
In NCO, POMO \citep{kwon2020pomo} and PolyNet \citep{hottung2025polynet} optimize
the \emph{mean} over diverse rollouts; Poppy \citep{grinsztajn2023poppy} and Leader
Reward \citep{wang2024leaderreward} optimize best-of-$K$ over \textbf{independent}
decoder populations / multi-starts, valid precisely because of independence
(Observation~\ref{obs:pomo}); Gumbeldore \citep{pirnay2024gumbeldore} trains from
WOR pools by self-improvement \emph{imitation}, not a policy gradient. In LLM
alignment, BOND \citep{sessa2024bond} and vBoN \citep{amini2025vbon} own
inference-aware best-of-$N$ under i.i.d.\ sampling; QuasiMoTTo
\citep{li2026quasimotto} runs correlated-sampler (QMC-lattice) policy-gradient RL
with an unbiased eval-side pass@$K$: a correlated sampler \emph{without} the
closed-form inclusion structure our weights need; CPPO \citep{li2026cppo} optimizes
a joint-policy pass@$K$ surrogate rather than an exogenous WOR Max@$K$ gradient;
Dr.GRPO \citep{liu2025drgrpo} owns the divide-by-std critique. Neither domain is
theoretically necessary to the note; both motivate it.

\section{Discussion and open problems}
\label{sec:discussion}

\begin{itemize}
\item \textbf{The bias sweep.} Whether the correction matters depends on the draw ratio
$K/n$ and the policy's heterogeneity; the illustrative sweep of
Appendix~\ref{app:biasmap} (Figure~\ref{fig:bias}) covers one reward/policy geometry
and is not a validated decision rule. A genuine
use-or-skip map (multiple $M,n$, random policy directions and reward geometries, with
medians and worst-case quantiles) is future work.
\item \textbf{Sequence scale.} Corollary~\ref{cor:sequence} packages the
sequence-scale case: the finite-support results hold for any finite, full-support,
normalized differentiable policy, and Corollary~\ref{cor:poolonly} shows the
estimator, DP, and sampler score require only the values and differentiable
computation graphs of the drawn sequences' probabilities plus
$\sum_y p_\theta(y)=1$; exact SBS supplies exactly those, and the support is never
enumerated. Countably infinite or unbounded variable-length support remains open:
there one must separately establish an integrable envelope for both the estimator and
the policy score.
\item \textbf{Why an estimator where full-support deterministic computation is available.} In the flat
enumerable regime the objective $\JW$ and its gradient admit the same one-dimensional
integral as \S\ref{sec:tractable}, run over the full item set with deterministic factors
$q_i=1$. Its finite-$Q$ evaluation costs $O(MKQ)$ and has zero Monte Carlo variance but
retains quadrature error. This is an algebraic
full-support specialization (or the $q_i\to1$ limit), not a threshold draw covered by
Theorem~\ref{thm:unbiased}, which requires $M\ge n+1$. The single-draw sample-reuse
estimator is therefore \emph{not} the
cheapest route to the gradient when the full support is affordable; its purpose is
compute amortization. It pays off when $M\gg n$ or when \emph{reward evaluation}
dominates: reusing the decoder's already-drawn $n$-item pool amortizes the $M-n$ reward
evaluations the exact sum would require and shrinks the per-step DP from $O(MKQ)$ to
$O(nKQ)$, trading those reward evaluations for the controlled variance of
Remark~\ref{rmk:variance}; per the diagnostic of Appendix~\ref{app:variance}
(Table~\ref{tab:variance}), that trade is favorable
only in specific regimes (practical baseline, oversampled pool). The raison d'\^etre is
fully realized only at structured / sequence scale, where full-support enumeration is
unavailable and the pool-only forms of Corollary~\ref{cor:poolonly} are the ones a
decoder can actually evaluate.
\item \textbf{Numerics.} The combinatorial collapse of Theorem~\ref{thm:tractable} is
algebraically \emph{exact}; the sole approximation is the Gauss--Laguerre quadrature of
the resulting one-dimensional integral. The strictly positive tail rate
$1-P_S+p_{\min,S}>0$ proves that the ideal integral is finite, including at $P_S=1$;
it is not, by itself, a certified finite-$Q$ Gauss--Laguerre error theorem. The
ancillary suite therefore reports $Q$-refinement and poly-vs-brute residuals on a
declared configuration grid, while a uniform value-and-gradient bound remains open. A
near-degenerate, highly concentrated policy makes the $1/\prod q$ prefactor large and
can overflow $h_i=(e^{p_it}-1)/q_i$ before the tail decays. Strict mode refuses such a
configuration; optional node dropping is explicitly biased.
\item \textbf{Toward a low-variance estimator.} The Horvitz--Thompson form is unbiased
but its second moment diverges in the near-exhaustive corner (Remark~\ref{rmk:variance}).
A natural research direction is to \emph{Rao--Blackwellize} it over the unordered pool set (the
Max@$K$ analogue of the unordered-set estimator of \citet{kool2020unordered}), which
forms no $1/\prod q$ weight. Defining that variant, proving its unbiasedness, and
benchmarking its variance belongs to separate companion work and is out of scope for
this theory-and-certification note.
\item \textbf{Benchmark.} A pre-registered, multi-seed, seed-level-inference NCO study
of Max@$K$ training under \emph{realized} best-of-$K$, the empirical companion, is
in preparation.
\end{itemize}

\paragraph{Reproducibility.}
The ancillary directory \texttt{anc/} contains the reference implementation, pinned
requirements, claim-to-test map, the finite-model certificates, the structured-SBS
certificate, and scripts for Figure~\ref{fig:bias} and Table~\ref{tab:variance}. From
that directory, \texttt{python3 -m pytest tests -q} runs the complete CPU suite. Strict
no-clamp/no-truncation numerics are the public scientific default; the explicitly biased
defensive mode is opt-in and tested separately (Appendix~\ref{app:implementation}). The
ancillary code is MIT-licensed (\texttt{anc/LICENSE}); the arXiv source archive serves
as the immutable code snapshot for this note, and a maintained copy lives in the
authors' repository,
\url{https://github.com/midhunxavier/rank-conditioned-sample-reuse}
(directory \texttt{anc/}; release tag
\texttt{coupled-maxk-arxiv-v1} pins this submission).

\appendix

\section{Full proof of Proposition~\ref{prop:gradient}}
\label{app:gradproof}

\begin{proof}[Proof of Proposition~\ref{prop:gradient}]
Since $\sgop$ changes only the autodiff derivative and not the value,
$\E[\widehat g]=\E\big[\nabla_\theta\Jhat+\Jhat\,\nabla_\theta\log f_\theta\big]$.
Put $\Phi(\theta):=\E_{f_\theta}[\Jhat(\pool,\kappa)]$. Theorem~\ref{thm:unbiased}
holds at \emph{every} full-support $\theta$, so $\Phi\equiv\JW$ as functions of
$\theta$ and hence $\nabla_\theta\Phi=\nabla_\theta\JW$. If differentiation under the
sampler expectation is licensed, the total-derivative (score-function) identity gives
\[
\nabla_\theta\Phi=\int\nabla_\theta(\Jhat f_\theta)
=\int f_\theta\big(\nabla_\theta\Jhat+\Jhat\,\nabla_\theta\log f_\theta\big)
=\E_{f_\theta}\big[\nabla_\theta\Jhat+\Jhat\,\nabla_\theta\log f_\theta\big],
\]
the middle equality holding pointwise because $f_\theta>0$ on its domain. The
proposition is therefore \emph{exactly} the claim that this interchange is valid; the
remainder builds an integrable dominating envelope that justifies it.

\emph{Exponential-clock law of the draw.} A realized draw is
$\omega=(P,m,\tau)$: an $n$-subset pool $P$, a threshold item $m\notin P$, and
$\tau=e^{-\kappa}=T_{(n+1)}\in(0,\infty)$, with reference measure ``counting on
$(P,m)$'' $\otimes$ ``Lebesgue $d\tau$''. In the exponential-clock coordinates of
Theorem~\ref{thm:unbiased}, $T_i=e^{-G_i}\sim\Expo(p_i)$ independently, the pool is the
$n$ smallest clocks, $\tau$ is the $(n{+}1)$-st, and
$q_i(\kappa)=\Pr(T_i<\tau)=1-e^{-p_i\tau}$. The joint density is
\[
f_\theta(P,m,\tau)=\underbrace{\prod_{i\in P}\big(1-e^{-p_i\tau}\big)}_{\text{pool clocks}<\tau}\;
\underbrace{p_m e^{-p_m\tau}}_{\text{threshold clock}=\tau}\;
\underbrace{\prod_{j\notin P\cup\{m\}}\!\!e^{-p_j\tau}}_{\text{far clocks}>\tau},\qquad\tau>0 .
\]
Its support, all $\binom Mn(M{-}n)$ pairs $(P,m)$ and all $\tau>0$, is fixed
independently of $\theta$, since the policy has full support. Thus $\nabla_\theta$
passes through the \emph{finite} $(P,m)$-sum termwise (this is point~(a): the discrete
part interchanges trivially), and only the $1$-D $\tau$-integral needs domination.
Passing to the note's $\kappa=-\log\tau$ multiplies $f_\theta$ by the $\theta$-free
Jacobian $\tau$, so $\nabla_\theta\log f_\theta$ is identical in either coordinate and
equals the sampler score of \S\ref{sec:prelim}.

\emph{The HT weight cancels the pool mass.} Substituting~($\star$) and
$q_i=1-e^{-p_i\tau}$, the $1/\prod_{i\in S}q_i$ blow-up of the estimator (as
$\tau\to0$, $q_i\sim p_i\tau$) is annihilated by the pool factor
$\prod_{i\in P}(1-e^{-p_i\tau})$ of $f_\theta$, the same cancellation that makes the
crux identity equal $1$:
\[
\Jhat\,f_\theta=\!\!\sum_{\substack{S\subseteq P\\|S|=K}}\!\!
\underbrace{\Big(\max_{i\in S}R_i\Big)\,p_m\,P_{\WOR}(S)\,
e^{-p_m\tau}\!\!\prod_{j\notin P\cup\{m\}}\!\!e^{-p_j\tau}\!\!\prod_{i\in P\setminus S}\!\!\big(1-e^{-p_i\tau}\big)}_{=:~\mathrm{Term}_S(P,m,\tau;\theta)} .
\]
Every factor of $\mathrm{Term}_S$ lies in $[0,1]$ except the constants $p_m\le p_{\max}$,
$P_{\WOR}(S)\le1$, $|\max_S R|\le R_{\max}$; and
$e^{-p_m\tau}\prod_{j\notin P\cup\{m\}}e^{-p_j\tau}=e^{-(1-P_P)\tau}$ with
$P_P=\sum_{i\in P}p_i$. Because $M\ge n+1$ there is at least one non-pool item, so the
decay rate is strictly positive: $1-P_P=\sum_{j\notin P}p_j\ge(M-n)p_{\min}>0$, whence
\begingroup\renewcommand{\theHequation}{eqenvA}
\begin{equation}
|\mathrm{Term}_S|\;\le\;p_{\max}R_{\max}\,e^{-(M-n)p_{\min}\,\tau}. \tag{A}\label{eq:envA}
\end{equation}
\endgroup

\emph{A dominating envelope on a compact neighborhood} (point~(b)). Fix a compact
neighborhood $\Theta_0\ni\theta$ small enough that
$p_i(\theta')\in[p_{\min},p_{\max}]\subset(0,1)$ for every item and every
$\theta'\in\Theta_0$; this is possible because the finite policy is continuous and has
full support at $\theta$. The finitely many maps $\theta'\mapsto p_i(\theta')$ are
continuously differentiable, so compactness gives one finite constant $L$ with
$\|\nabla_{\theta'}p_i\|\le L$ uniformly over $i$ and $\Theta_0$.
$\mathrm{Term}_S$ depends on $\theta'$ through the prefactors $p_m$ and $P_{\WOR}(S)$ and
through the exponential factors $e^{-p_i\tau}$, $(1-e^{-p_i\tau})$ (the reward
$\max_S R$ is constant). The prefactors are smooth on $\Theta_0$ with uniformly bounded
derivatives, $\|\nabla_{\theta'}p_m\|\le L$ and $\|\nabla_{\theta'}P_{\WOR}(S)\|\le C_{PL}$
($P_{\WOR}$ is a finite rational function of the $p_i$ whose Plackett--Luce denominators
stay $\ge(M-K+1)p_{\min}>0$ on $\Theta_0$), so their product-rule contributions carry no
factor of $\tau$ and populate the $\tau$-independent part of the envelope. Differentiating
the exponential factors, by the product rule,
$\partial_{\theta'_\ell}e^{-p_i\tau}=-\tau(\partial_{\theta'_\ell}p_i)e^{-p_i\tau}$ and
$\partial_{\theta'_\ell}(1-e^{-p_i\tau})=\tau(\partial_{\theta'_\ell}p_i)e^{-p_i\tau}$,
pulls down a single factor $\tau$ times a coefficient bounded by $L$, while the
remaining factors stay in $[0,1]$ (or are bounded by the constants $p_{\max},R_{\max}$).
Summing the $O(M)$ product-rule contributions and using~\eqref{eq:envA} for the
undifferentiated envelope, there is a constant $C=C(M,n,R_{\max},L,p_{\max},C_{PL})$,
uniform over $\Theta_0$ (with $C$ and the decay rate depending on the ball $\Theta_0$
through its bounds $p_{\min},p_{\max}$; any ball around the interior point $\theta$
suffices), with
\[
\sup_{\theta'\in\Theta_0}\big\|\nabla_{\theta'}\mathrm{Term}_S(P,m,\tau;\theta')\big\|
\;\le\;C\,(1+\tau)\,e^{-(M-n)p_{\min}\tau}\;=:\;g(\tau).
\]
Summing over the finitely many $(S,P,m)$ (at most $\binom Mn(M-n)\binom nK$ of them)
yields a single $\theta'$-independent bound
$\big\|\nabla_{\theta'}(\Jhat f_{\theta'})\big\|\le\binom Mn(M-n)\binom nK\,g(\tau)$
valid throughout $\Theta_0$. Since
$\int_0^\infty(1+\tau)\,e^{-(M-n)p_{\min}\tau}\,d\tau<\infty$, this envelope is
Lebesgue-integrable in $\tau$ and independent of $\theta'$ on $\Theta_0$. The standard
Leibniz rule for differentiating under the integral sign (dominated convergence)
therefore licenses
\[
\nabla_{\theta'}\!\int_0^\infty(\Jhat f_{\theta'})\,d\tau
=\int_0^\infty\nabla_{\theta'}(\Jhat f_{\theta'})\,d\tau
\qquad\text{for all }\theta'\in\Theta_0,
\]
in particular at $\theta'=\theta$. Combined with the termwise discrete $(P,m)$-sum of
point~(a), this is precisely the interchange
$\nabla_\theta\Phi=\E_{f_\theta}[\nabla_\theta\Jhat+\Jhat\,\nabla_\theta\log f_\theta]$
that was the sole remaining gap. With $\nabla_\theta\Phi=\nabla_\theta\JW$ from the
first paragraph, we conclude
$\E[\widehat g]=\E[\nabla_\theta\Jhat+\Jhat\,\nabla_\theta\log f_\theta]=\nabla_\theta\JW$.
\end{proof}

\section{Proof of Proposition~\ref{prop:gradmoment}}
\label{app:gradmoment}

Work in the exponential-clock coordinates and notation of
Appendix~\ref{app:gradproof}: a realized draw is $\omega=(P,m,\tau)$ with pool $P$
($|P|=n$), threshold item $m$, and threshold clock $\tau=e^{-\kappa}\in(0,\infty)$;
$q_i=1-e^{-p_i\tau}$; and $f_\theta(P,m,\tau)$ is the joint density displayed there.
Because the $\kappa\mapsto\tau$ change of coordinates has a $\theta$-free Jacobian,
$\nabla_\theta\log f_\theta$ is the same in either coordinate. Throughout,
$\nabla_\theta$ acts at fixed $(P,m,\tau)$ through $p(\theta)$ only (the
detached-threshold convention of \S\ref{sec:prelim}), and, value-wise,
\[
\widehat g=\nabla_\theta\Jhat+\Jhat\,\nabla_\theta\log f_\theta,
\qquad
\Jhat=\sum_{\substack{S\subseteq P\\|S|=K}}\Big(\max_{i\in S}R_i\Big)
P_{\WOR}(S)\,W_S,
\qquad
W_S=\frac{\ind\{S\subseteq P\}}{\prod_{i\in S}q_i}.
\]
Fix $\theta$; let $p_{\min}=\min_ip_i>0$, $p_{\max}=\max_ip_i<1$, and
$L=\max_i\|\nabla_\theta\log p_i\|<\infty$ (finitely many continuously
differentiable coordinates at a fixed interior point).

\begin{lemma}[uniform score bounds]
\label{lem:scorebounds}
For every $\tau>0$, every item $i$, every $K$-subset $S$, and every configuration
$(P,m)$:
\begin{itemize}
\item[(i)] $\nabla_\theta\log q_i=\varphi(p_i\tau)\,\nabla_\theta\log p_i$ with
$\varphi(x)=x/(e^x-1)\in(0,1)$ for $x>0$; hence
$\sup_{\tau>0}\|\nabla_\theta\log q_i\|\le L$.
\item[(ii)] $\|\nabla_\theta\log P_{\WOR}(S)\|\le 2KL$.
\item[(iii)] $\|\nabla_\theta\log f_\theta(P,m,\tau)\|
\le L\big(n+1+\tau(1-P_P)\big)\le L\,(n+1+\tau)$, where $P_P=\sum_{i\in P}p_i$.
\end{itemize}
\end{lemma}

\begin{proof}
(i) $\nabla_\theta q_i=\tau e^{-p_i\tau}\,\nabla_\theta p_i
=p_i\tau e^{-p_i\tau}\,\nabla_\theta\log p_i$; dividing by $q_i=1-e^{-p_i\tau}$
gives the scalar prefactor
$p_i\tau e^{-p_i\tau}/(1-e^{-p_i\tau})=p_i\tau/(e^{p_i\tau}-1)=\varphi(p_i\tau)$.
For $x>0$, $e^x-1>x$ gives $0<\varphi(x)<1$. The bound is uniform in $\tau$: the
$1/q_i$ blow-up as $\tau\to0$ affects the weight $W_S$, never the score
$\nabla_\theta\log q_i$.

(ii) Write $P_{\WOR}(S)=\sum_\sigma g_\sigma$ over the $K!$ orderings $\sigma$ of
$S$, with $g_\sigma=\prod_{j=1}^{K}p_{\sigma_j}/D_{\sigma,j}$ and
$D_{\sigma,j}=1-\sum_{l<j}p_{\sigma_l}
=\sum_{i\notin\{\sigma_1,\dots,\sigma_{j-1}\}}p_i>0$ (the remaining mass, a
nonempty sum since $K<M$). Each denominator's log-gradient is a convex combination
of item scores,
\[
\nabla_\theta\log D_{\sigma,j}
=\!\!\sum_{i\notin\{\sigma_1,\dots,\sigma_{j-1}\}}\!\!\frac{p_i}{D_{\sigma,j}}\,
\nabla_\theta\log p_i,
\qquad
\big\|\nabla_\theta\log D_{\sigma,j}\big\|\le L,
\]
so $\|\nabla_\theta\log g_\sigma\|\le\sum_{j=1}^{K}(L+L)=2KL$. A finite sum of
positive terms preserves the bound, since
$\nabla_\theta\log\sum_\sigma g_\sigma
=\sum_\sigma\big(g_\sigma/\sum_{\sigma'}g_{\sigma'}\big)\,
\nabla_\theta\log g_\sigma$ is again a convex combination. Hence
$\|\nabla_\theta\log P_{\WOR}(S)\|\le2KL$.

(iii) From the density in Appendix~\ref{app:gradproof},
$\log f_\theta=\sum_{i\in P}\log q_i+\log p_m-p_m\tau
-\tau\sum_{j\notin P\cup\{m\}}p_j$. By (i),
$\nabla_\theta p_j=p_j\nabla_\theta\log p_j$, and the triangle inequality,
\[
\|\nabla_\theta\log f_\theta\|
\le nL+L+\tau L\Big(p_m+\!\!\sum_{j\notin P\cup\{m\}}\!\!p_j\Big)
=L\big(n+1+\tau(1-P_P)\big)\le L(n+1+\tau). \qedhere
\]
\end{proof}

\begin{lemma}[master pointwise bound]
\label{lem:master}
At every realized $(P,m,\tau)$,
\[
\|\widehat g\|\;\le\;R_{\max}L\,(3K+n+1+\tau)\sum_{|S|=K}P_{\WOR}(S)\,W_S .
\]
\end{lemma}

\begin{proof}
Differentiating one term of $\Jhat$ at fixed $(P,m,\tau)$ and applying
Lemma~\ref{lem:scorebounds}(i)--(ii),
\[
\Big\|\nabla_\theta\big(P_{\WOR}(S)\,W_S\big)\Big\|
=P_{\WOR}(S)\,W_S\,\Big\|\nabla_\theta\log P_{\WOR}(S)
-\sum_{i\in S}\nabla_\theta\log q_i\Big\|
\le 3KL\,P_{\WOR}(S)\,W_S,
\]
so $\|\nabla_\theta\Jhat\|\le3KLR_{\max}\sum_{S}P_{\WOR}(S)W_S$ (sums over all
$|S|=K$; a term with $S\not\subseteq P$ is identically zero as a function of
$\theta$ at this realization, so it contributes neither value nor gradient). With
$|\Jhat|\le R_{\max}\sum_SP_{\WOR}(S)W_S$ and Lemma~\ref{lem:scorebounds}(iii),
$\|\widehat g\|\le\|\nabla_\theta\Jhat\|+|\Jhat|\,\|\nabla_\theta\log f_\theta\|$
gives the claim. The three score bounds and this master inequality are verified
pointwise by autodiff on randomized draws, including thresholds far outside the
typical sampling range on both sides, in
\texttt{anc/tests/test\_gradient\_moment.py}.
\end{proof}

\begin{lemma}[weighted per-term moments]
\label{lem:weightedmoment}
Fix $|S|=K$ and a constant $a\ge0$. If $n\ge2K$ (and $M\ge n+1$) then
$\E\big[(a+\tau)^2W_S^2\big]<\infty$.
\end{lemma}

\begin{proof}
Condition on the background clocks $B=\{T_j:j\notin S\}$ as in
Theorem~\ref{thm:unbiased} and Remark~\ref{rmk:variance}: on $\{S\subseteq P\}$
the threshold equals $b:=b_{n-K+1}$, the $(n{-}K{+}1)$-st smallest background
clock, which is $B$-measurable; the $S$-clocks are independent of $B$, and off the
inclusion event $W_S=0$. Hence
\[
\E\big[(a+\tau)^2W_S^2\,\big|\,B\big]
=(a+b)^2\,\frac{\Pr\big(S\subseteq P\mid B\big)}
{\prod_{i\in S}\big(1-e^{-p_ib}\big)^{2}}
=\frac{(a+b)^2}{\prod_{i\in S}\big(1-e^{-p_ib}\big)},
\]
and it suffices to bound
$\E_B\big[(a+b)^2\prod_{i\in S}(1-e^{-p_ib})^{-1}\big]$. Split at
$t_0:=1/p_{\max}\ (\ge1)$.

\emph{Near zero ($b\le t_0$).} Here $p_ib\le1$, and $1-e^{-x}\ge x-x^2/2\ge x/2$
on $[0,1]$, so $\prod_{i\in S}(1-e^{-p_ib})^{-1}\le 2^K p_{\min}^{-K}\,b^{-K}$.
The variable $b$ is the $(n{-}K{+}1)$-st order statistic of the $M-K$ independent
background clocks $T_j\sim\Expo(p_j)$; for $t\le t_0$ its density obeys
\begin{align*}
f_b(t)\;&\le\;\sum_{j\notin S}p_je^{-p_jt}\,
\Pr\big(\text{exactly $n-K$ of the others}<t\big)\\
&\le\;(M-K)\,p_{\max}\binom{M-K-1}{n-K}(p_{\max}t)^{\,n-K}
\;=:\;C_1\,t^{\,n-K},
\end{align*}
using $1-e^{-p_jt}\le p_jt\le p_{\max}t$ for each of the $n-K$ below-threshold
factors and $\le1$ for the survival factors. The near-zero contribution is
therefore at most
\[
(a+t_0)^2\,2^K p_{\min}^{-K}\,C_1\int_0^{t_0}t^{\,n-2K}\,dt\;<\;\infty,
\]
the integral converging precisely because $n\ge2K$ makes the exponent
nonnegative.

\emph{Tail ($b>t_0$).} Every factor obeys
$1-e^{-p_ib}\ge1-e^{-p_{\min}t_0}>0$, so
$\prod_{i\in S}(1-e^{-p_ib})^{-1}\le(1-e^{-p_{\min}t_0})^{-K}=:C_2$. Moreover
$b>t$ forces at least $(M-K)-(n-K)=M-n\ge1$ background clocks to exceed $t$, so a
union bound over which clocks those are gives
$\Pr(b>t)\le\binom{M-K}{M-n}\,e^{-(M-n)p_{\min}t}$: the threshold has an
exponential upper tail, hence $\E[(a+b)^2]<\infty$ and the tail contribution is at
most $C_2\,\E[(a+b)^2]<\infty$.

Combining the two ranges proves the lemma.
\end{proof}

\begin{proof}[Proof of Proposition~\ref{prop:gradmoment}]
Square Lemma~\ref{lem:master} and use
$(3K+n+1+\tau)^2\le2\big(9K^2+(n+1+\tau)^2\big)$:
\[
\|\widehat g\|^2\;\le\;2R_{\max}^2L^2\,\big(9K^2+(n+1+\tau)^2\big)
\Big(\sum_{|S|=K}P_{\WOR}(S)\,W_S\Big)^{\!2}.
\]
The weights $P_{\WOR}(\cdot)$ form a probability distribution over the
$K$-subsets, $\sum_{|S|=K}P_{\WOR}(S)=1$, so Jensen's inequality gives
$\big(\sum_SP_{\WOR}(S)W_S\big)^2\le\sum_SP_{\WOR}(S)\,W_S^2$ pointwise, and
therefore
\begin{align*}
\E\|\widehat g\|^2
\;&\le\;2R_{\max}^2L^2\sum_{|S|=K}P_{\WOR}(S)\,
\E\Big[\big(9K^2+(n+1+\tau)^2\big)W_S^2\Big]\\
&\le\;2R_{\max}^2L^2\max_{|S|=K}
\E\Big[\big(9K^2+(n+1+\tau)^2\big)W_S^2\Big].
\end{align*}
Each expectation on the right is finite when $n\ge2K$: apply
Lemma~\ref{lem:weightedmoment} with $a=n+1$ for the $(n+1+\tau)^2$ part, and note
$9K^2\,\E[W_S^2]\le9K^2\,\E[(1+\tau)^2W_S^2]<\infty$ by the same lemma with $a=1$.
This proves the displayed bound and the proposition.
\end{proof}

\section{Second-moment simulation details (Remark~\ref{rmk:variance})}
\label{app:secondmoment}

We confirmed this numerically with the exact Gumbel-Top-$n$ sampler (near-uniform \emph{and} heavily skewed logits, $M=2n$). Two thresholds must be kept apart. The \emph{true value} $\E[W_S^2]$ is finite iff $n\ge2K$ (the boundary the figure and theorem rely on); but the plug-in Monte-Carlo \emph{estimate} of $\E[W_S^2]$ has a finite sampling variance only for $n\ge4K$, since $\operatorname{Var}(W_S^2)$ involves $\E[W_S^4]\sim\int_{0^+}\tau^{\,n-4K}\,d\tau$, which diverges for $n<4K$ (equivalently $W_S^2$ has tail index $(n{+}1)/(2K)$, which must exceed $2$).

So on the marginal cells $n=2K$ [e.g.\ $(n,K)=(4,2),(6,3),(8,4),(10,5)$] the value is finite yet consistent-but-heavy-tailed to estimate, so its running mean converges slowly and cannot be pinned to three significant figures.

What the run shows cleanly is the \emph{qualitative} separation: a finite $\E[W_S^2]\sim O(10\text{--}10^2)$ versus realized weights $1/\prod_{i\in S}q_i$ exceeding $10^{4}$ in the divergent $n<2K$ cells [e.g.\ $(4,3),(4,4),(6,4),(8,5)$]. Skew inflates the constant but never changes which side of $n=2K$ a cell lands on.

\section{Gradient-variance diagnostic: protocol and full table}
\label{app:variance}

This appendix holds the protocol and the full numbers behind the variance bullet of
\S\ref{sec:numerics}. It is a deliberately small sanity check---one family of random
reward/logit geometries, $K=2$ only---sized to confirm the qualitative regime
statements quoted there, not to establish a general variance ranking. ``HT'' below
means the $Q=96$ implementation of the theoretically unbiased ideal estimator; its
finite-$Q$ residual is negligible on these certified cells but is not identically zero.

\emph{Protocol.} $K=2$ throughout, pools $n\in\{4,6,8\}$ (draw ratios
$K/n\in\{\tfrac12,\tfrac13,\tfrac14\}$); five random reward/logit geometries per
configuration $\times$ three Monte-Carlo seeds; $4{,}000$ draws per cell; entries are
medians [IQR] across the $15$ cells. Arms follow the naming of \S\ref{sec:where}: PL
joint-score REINFORCE (unbiased), our threshold-conditioned HT (ideal form unbiased;
finite-$Q$ implementation used), and the shared-value score surrogate (biased). Both baselines target the same quantity
$\JW$: the oracle uses its exact value (enumerable-only), the practical
EMA estimates it online---so the two panels differ only in oracle vs.\ estimated
baseline, and they agree.

\emph{Cost normalization.} Entries are the single-draw gradient variance
$\operatorname{tr}\operatorname{Cov}$ \emph{times reward-evaluations per draw}: the
joint-score arm consumes a fresh size-$K$ draw ($K$ reward evaluations), the pool
arms a size-$n$ pool ($n$ evaluations). This is the appropriate unit when reward
evaluation dominates the step cost, the regime the estimator targets
(\S\ref{sec:discussion}); it deliberately excludes estimator wall-clock, which at
these sizes runs \emph{against} the HT arm: $0.8$--$1.9$~ms/draw (the DP) vs.\
$\approx0.1$~ms for the other arms.

\emph{Tails.} At $K/n{=}\tfrac12$ the HT arm is heavy-tailed (median max single-draw
gradient norm $13.5$ vs.\ $1.9$ for joint-score; $95$th percentiles comparable), as
Remark~\ref{rmk:variance} predicts at the finite-variance boundary.

Reproduction: \texttt{anc/src/experiments/variance\_comparison.py} ($\approx$20--40
min single-core; prints per-cell bias, variance, tail, and ms/draw diagnostics).

\begin{table}[!tbp]
\centering
\caption{\textbf{Finite-model gradient-estimator diagnostic, cost-normalized} ($K=2$;
median [IQR] across the $15$ cells; protocol, cost-normalization caveat, and tail
diagnostics in the text of this appendix).}
\label{tab:variance}
\begin{tabular}{lccc}
\toprule
& \multicolumn{3}{c}{cost-normalized gradient variance at $K/n=$}\\
\cmidrule(l){2-4}
estimator & $0.50$ & $0.33$ & $0.25$ \\
\midrule
\multicolumn{4}{l}{\emph{EMA (practical) baseline, tracking $\JW$}}\\
PL joint-score REINFORCE (unbiased) & $0.33$ {\scriptsize$[0.26,0.34]$} & $1.49$ {\scriptsize$[0.63,1.90]$} & $0.68$ {\scriptsize$[0.33,0.70]$} \\
HT, threshold-conditioned (ours)$^\ddagger$ & $1.88$ {\scriptsize$[0.26,2.83]$} & $0.66$ {\scriptsize$[0.55,1.95]$} & $0.43$ {\scriptsize$[0.23,1.13]$} \\
shared-value score \emph{(biased)} & $0.26^\dagger$ & $0.96^\dagger$ & $0.81^\dagger$ \\
\midrule
\multicolumn{4}{l}{\emph{Oracle baseline $\JW$ (enumerable-only)}}\\
PL joint-score REINFORCE (unbiased) & $0.32$ {\scriptsize$[0.26,0.33]$} & $1.45$ {\scriptsize$[0.61,1.88]$} & $0.67$ {\scriptsize$[0.33,0.68]$} \\
HT, threshold-conditioned (ours)$^\ddagger$ & $1.81$ {\scriptsize$[0.24,2.79]$} & $0.60$ {\scriptsize$[0.33,1.57]$} & $0.31$ {\scriptsize$[0.09,0.96]$} \\
shared-value score \emph{(biased)} & $0.22^\dagger$ & $0.62^\dagger$ & $0.77^\dagger$ \\
\bottomrule
\end{tabular}
\\[3pt]
{\footnotesize $^\dagger$ low variance around a \emph{wrong} gradient (median relative bias $0.6$--$1.2$), so not comparable. $^\ddagger$Ideal estimator unbiased; table uses the certified $Q=96$ approximation.}
\end{table}

\section{Illustrative bias decomposition (Figure~\ref{fig:bias})}
\label{app:biasmap}

This appendix holds the illustrative $(K/n,\text{heterogeneity})$ sweep referenced
in \S\ref{sec:where}. It is one toy geometry ($M=8$, $n=6$), retained as an
illustration of \emph{how} the naive i.i.d.\ weights fail under coupling, not as a
decision rule (\S\ref{sec:discussion}); the analytical proof of the failure is the
closed-form instance of Observation~\ref{obs:coupling}.

Figure~\ref{fig:bias} decomposes the naive iid-SubLOO error into three panels.
\emph{Direction} (panel a): the naive weights stay aligned with the true gradient
everywhere (cosine $\ge0.926$). \emph{Magnitude} (panel b): they are substantially
off across most of the plane (relative $L_2$ median $\approx0.82$, with $60\%$ of
the $65$ cells exceeding $0.5$) and overshoot hardest where the draw is
near-exhaustive ($K/n\to1$) \emph{and} the policy is skewed. \emph{Absolute error}
(panel c) stays small even in that corner: the relative-$L_2$ values up to $\sim75$
there arise because $\|\nabla_\theta\JW\|\to0$ (hatched, denominator-unstable
cells), not because the naive step is absolutely large. The corrected
estimator~($\star$) restores the \emph{full} gradient (direction and magnitude) and
is unbiased everywhere on the map; in this toy geometry the dominant naive-weight
error happens to be magnitude, but on other substrates the directional distortion
can be larger.

\begin{figure}[!tbp]
\centering
\includegraphics[width=0.98\linewidth]{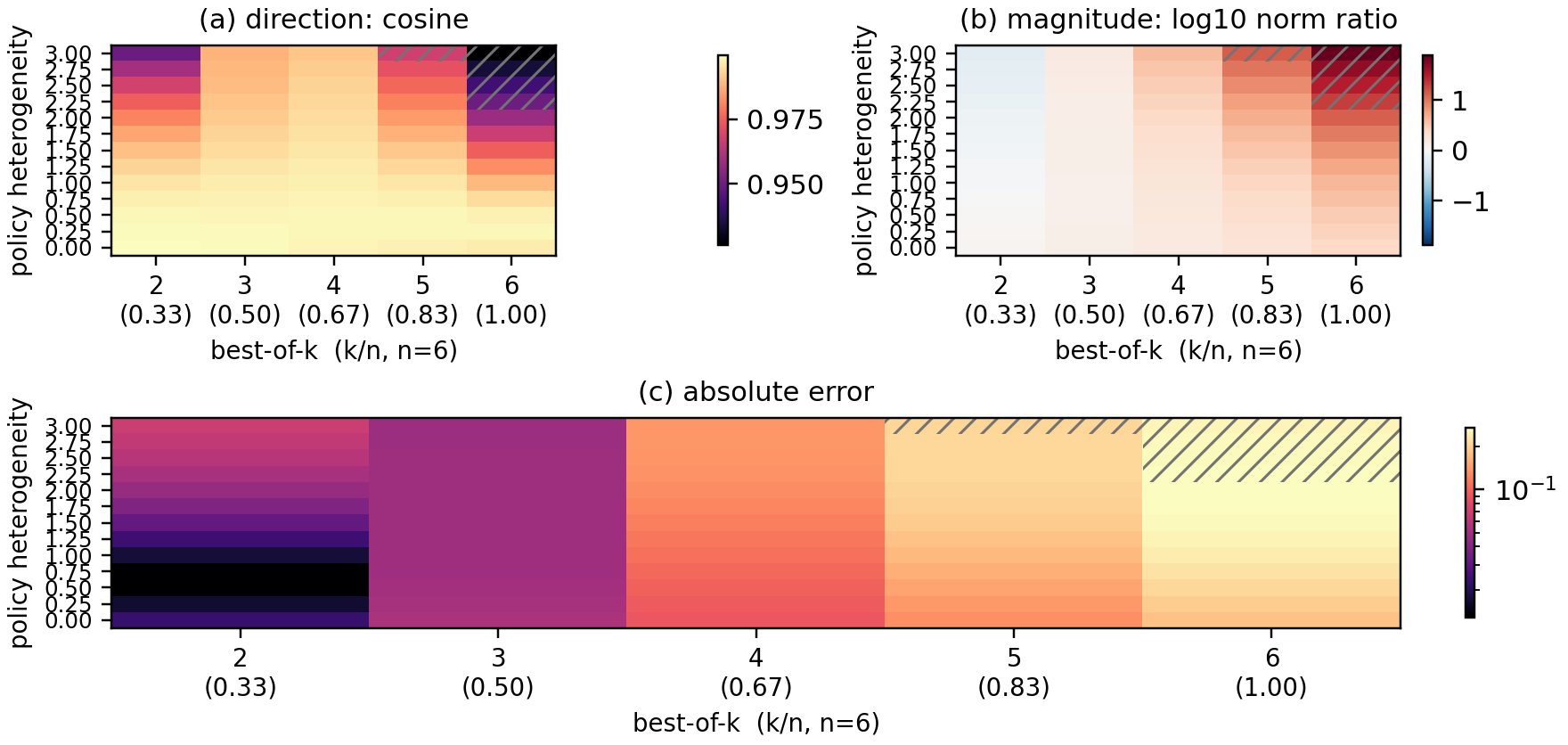}
\caption{\textbf{Illustrative bias decomposition} ($M=8,n=6$, one geometry).
For naive iid-SubLOO, \textbf{(a)} direction remains close
($\cos\ge0.926$), \textbf{(b)} magnitude overshoots near the
exhaustive/skewed corner, and \textbf{(c)} absolute error remains at most $0.6$.
Hatching marks denominator-unstable cells with
$\|\nabla_\theta\JW\|$ below $10\%$ of its map maximum. This is not a general regime
map: the ideal estimator~($\star$) is unbiased throughout, while
Remark~\ref{rmk:variance} identifies the term-level variance boundary.}
\label{fig:bias}
\end{figure}

\section{Implementation conventions: threshold detach and numerical modes}
\label{app:implementation}

\paragraph{The threshold must be stop-gradiented.}
A Gumbel-Top-$n$ sampler returns $\kappa=G_{(n+1)}=\phi_{(n+1)}+\Gumbel$, which
carries a gradient to $\theta$. The public loss internally stop-gradients it
(\texttt{kappa = kappa.detach()}), so autograd implements the fixed-threshold
score-function decomposition of Proposition~\ref{prop:gradient}; with finite
quadrature it applies that decomposition to $\JhatQ$, the numerical approximation
of \S\ref{sec:tractable}, rather than to the ideal $\Jhat$. A \emph{raw,
undetached} construction differs: autograd then adds an extra
pathwise-through-threshold term absent from the proof. A test in the ancillary
suite (\texttt{anc/}) pins both facts: the raw undetached gradient differs from the
detached one, and the public loss reproduces the detached gradient regardless of
whether $\kappa$ arrives with grad.

\paragraph{Strict vs.\ defensive numerics.}
The implementation exposes two numerical modes. The scientific default is
\emph{strict}: it has no score clamps or quadrature-node truncation and raises an
error rather than biasing. An opt-in \emph{defensive} mode can clamp or truncate
flagged negligible-density regimes and is explicitly biased when those guards
engage (the $1$-D integral itself remains finite-$Q$ quadrature in both modes). On
every validated configuration the guards do not engage, so the two modes coincide
there (checked by a dedicated test).

\bibliographystyle{plainnat}
\bibliography{references}

\end{document}